\documentclass[11pt]{article}
\usepackage[thmeqnumber]{matus}
\setlist[enumerate]{leftmargin=2em}
\usepackage{algorithm}
\usepackage{algorithmic}
\setlist[itemize]{leftmargin=2em}

\usepackage{caption}
\def\tv{\tilde{v}}

\def\tv{{\textsc{tv}}}

\makeatletter
\DeclareFontFamily{OMX}{MnSymbolE}{}
\DeclareSymbolFont{MnLargeSymbols}{OMX}{MnSymbolE}{m}{n}
\SetSymbolFont{MnLargeSymbols}{bold}{OMX}{MnSymbolE}{b}{n}
\DeclareFontShape{OMX}{MnSymbolE}{m}{n}{
    <-6>  MnSymbolE5
   <6-7>  MnSymbolE6
   <7-8>  MnSymbolE7
   <8-9>  MnSymbolE8
   <9-10> MnSymbolE9
  <10-12> MnSymbolE10
  <12->   MnSymbolE12
}{}
\DeclareFontShape{OMX}{MnSymbolE}{b}{n}{
    <-6>  MnSymbolE-Bold5
   <6-7>  MnSymbolE-Bold6
   <7-8>  MnSymbolE-Bold7
   <8-9>  MnSymbolE-Bold8
   <9-10> MnSymbolE-Bold9
  <10-12> MnSymbolE-Bold10
  <12->   MnSymbolE-Bold12
}{}

\let\llangle\@undefined
\let\rrangle\@undefined
\DeclareMathDelimiter{\llangle}{\mathopen}{MnLargeSymbols}{'164}{MnLargeSymbols}{'164}
\DeclareMathDelimiter{\rrangle}{\mathclose}{MnLargeSymbols}{'171}{MnLargeSymbols}{'171}
\makeatother

\newenvironment{proofof}[1]{\begin{proof}[Proof of {#1}]}{\end{proof}}
\def\tell{\widetilde\ell}

\def\Li{\textup{Li}}
\renewcommand{\tv}{\textsc{tv}}
\renewcommand{\Pr}{\textsc{Pr}}
\newcommand{\Ex}{\textsc{Ex}}

\def\wref{w_{\scriptstyle{\textup{ref}}}}
\def\uref{u_{\scriptstyle{\textup{ref}}}}
\def\tref{t_{\scriptstyle{\textup{ref}}}}

\def\td{G}

\def\cE{\mathcal{E}}
\def\hcR{\widehat{\cR}}
\def\yhat{{\hat y}}
\def\sgn{\textup{sgn}}

\title{\textbf{Stochastic linear optimization never overfits\\{}with quadratically-bounded losses on general data}}

\author{Matus Telgarsky \href{mailto:mjt@illinois.edu}{\texttt{<mjt@illinois.edu>}}}

\date{}

\begin{document}

\maketitle

\begin{abstract}This work provides test error bounds for iterative fixed point methods on linear
  predictors~---~specifically, stochastic and batch mirror descent (MD),
  and stochastic temporal difference learning (TD)~---~with
  two core contributions:
  (a)
  a single proof technique which gives high probability guarantees despite the absence of projections,
  regularization, or any equivalents, even when optima have large or infinite norm,
  for quadratically-bounded losses (e.g., providing unified treatment of squared and logistic losses);
  (b) locally-adapted rates which depend not on global problem structure (such as condition numbers
  and maximum margins),
  but rather on properties of low norm predictors which may suffer some small excess test error.
  The proof technique is an elementary and versatile coupling argument,
  and is demonstrated here in the following settings:
  stochastic MD under realizability;
  stochastic MD for general Markov data;
  batch MD for general IID data;
  stochastic MD on heavy-tailed data (still without projections);
  stochastic TD on Markov chains (all prior stochastic TD bounds are in
  expectation).
\end{abstract}

\section{Introduction}

This work studies iterative fixed point methods resembling (stochastic)
gradient descent, specifically
\begin{align}
  \textup{gradient descent (GD)},
  &&
  w_{i+1}
  &:= w_i - \eta g_{i+1},
  \label{eq:gd:1}
  \\
  \textup{mirror descent (MD)},
  &&
  w_{i+1}
  &:= \argmin\cbr{  \ip{\eta g_{i+1}}{w} + D_\psi(w,w_i) \ : \ w\in S},
  \label{eq:md:1}
  \\
  \textup{temporal difference learning (TD)},
  &&
  w_{i+1}
  &:= w_i - \eta \td_{i+1}(w_i),
  \label{eq:td:1}
\end{align}
where $g_{i+1}$ is stochastic or batch gradient,
$\td_{i+1}$ is a superficially similar
affine mapping related to the Bellman error in reinforcement learning,
and $D_\psi$ is a Bregman divergence.  (Details will come in \Cref{sec:notation}.)

The goal here is to control the excess risk of these procedures with high probability,
meaning to ensure that $\frac 1 t \sum_{i<t} \cR(w_i) - \cR(\wref)$ is small,
where the risk $\cR$ in the simplest setting is $\cR(w) = \Ex_{x,y} \ell(y, w^\T x)$
with $\ell$ a \emph{quadratically-bounded loss}, which roughly speaking means $\|\partial_w \ell(y,w^\T x)\|_*$ 
can be related to $\|w - \wref\|$ (cf. \Cref{sec:notation}), and $\wref$ is some good (but not necessarily optimal)
\emph{reference solution}.  So far, this is not too outlandish, however the focus in this
work is (a) a single proof technique for all methods and settings without
projections (e.g., $S = \R^d$ in MD in \cref{eq:md:1}),
constraints, or regularization, despite the possibility of all minimizers being at infinity
or simply too large,
as is often the case in practice,
and
(b) rates which depend not on global structure, but rather on the behavior of reasonably good but low norm
predictors.  In more detail,
these two high-level contributions and their relationship with prior work is as follows.

\begin{figure}[t]
 \begin{tcolorbox}[enhanced jigsaw, empty, sharp corners, colback=white,borderline north = {1pt}{0pt}{black},borderline south = {1pt}{0pt}{black},left=0pt,right=0pt,boxsep=0pt,rightrule=0pt,leftrule=0pt]
  \centering
  \begin{subfigure}[t]{0.48\textwidth}
    \centering
    \includegraphics[width = \textwidth]{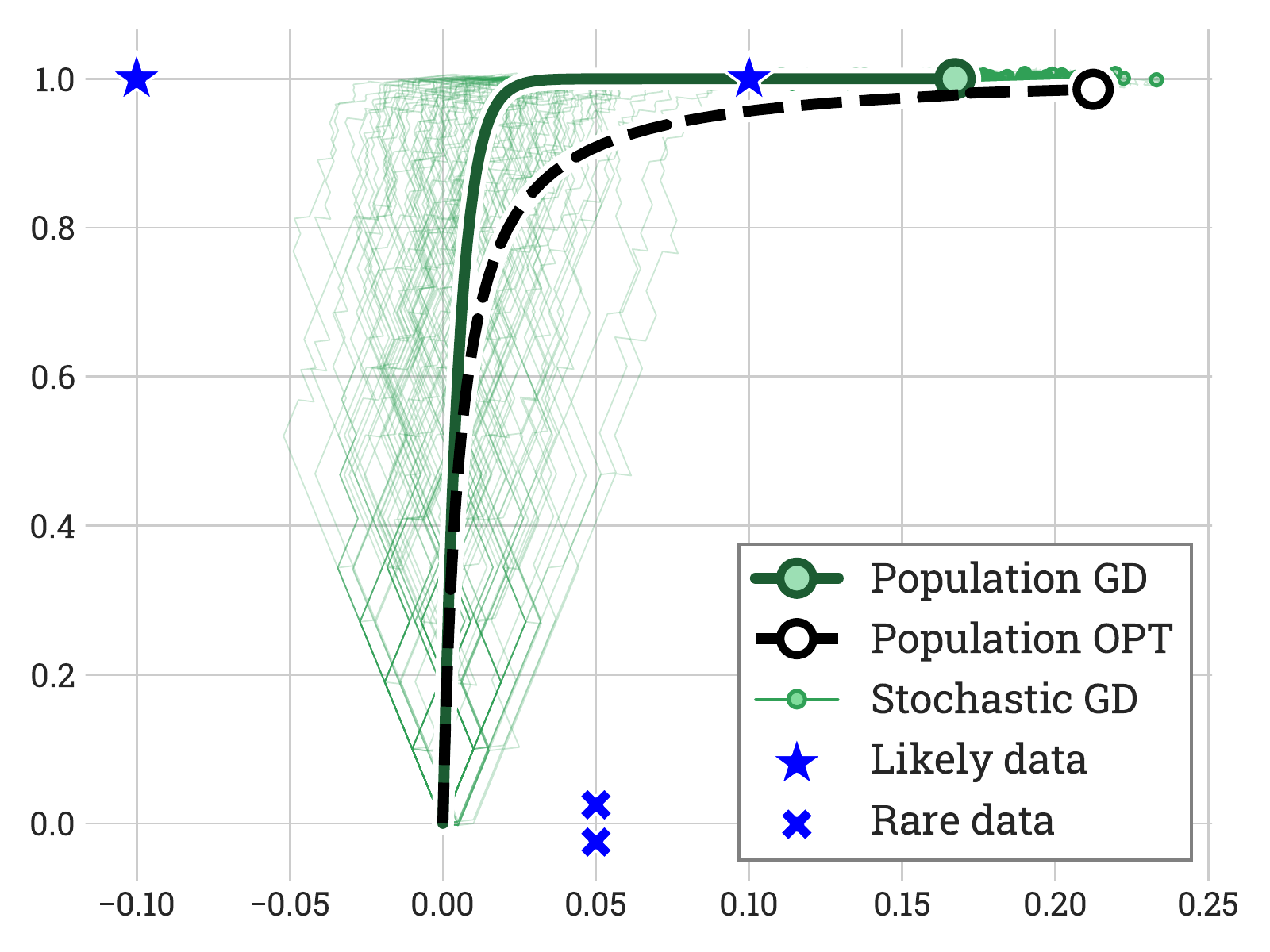}
    \caption{Squared loss with step size $\eta = 0.1$.}
    \label{fig:lsq}
  \end{subfigure}
  \hfill
  \begin{subfigure}[t]{0.48\textwidth}
    \centering
    \includegraphics[width = \textwidth]{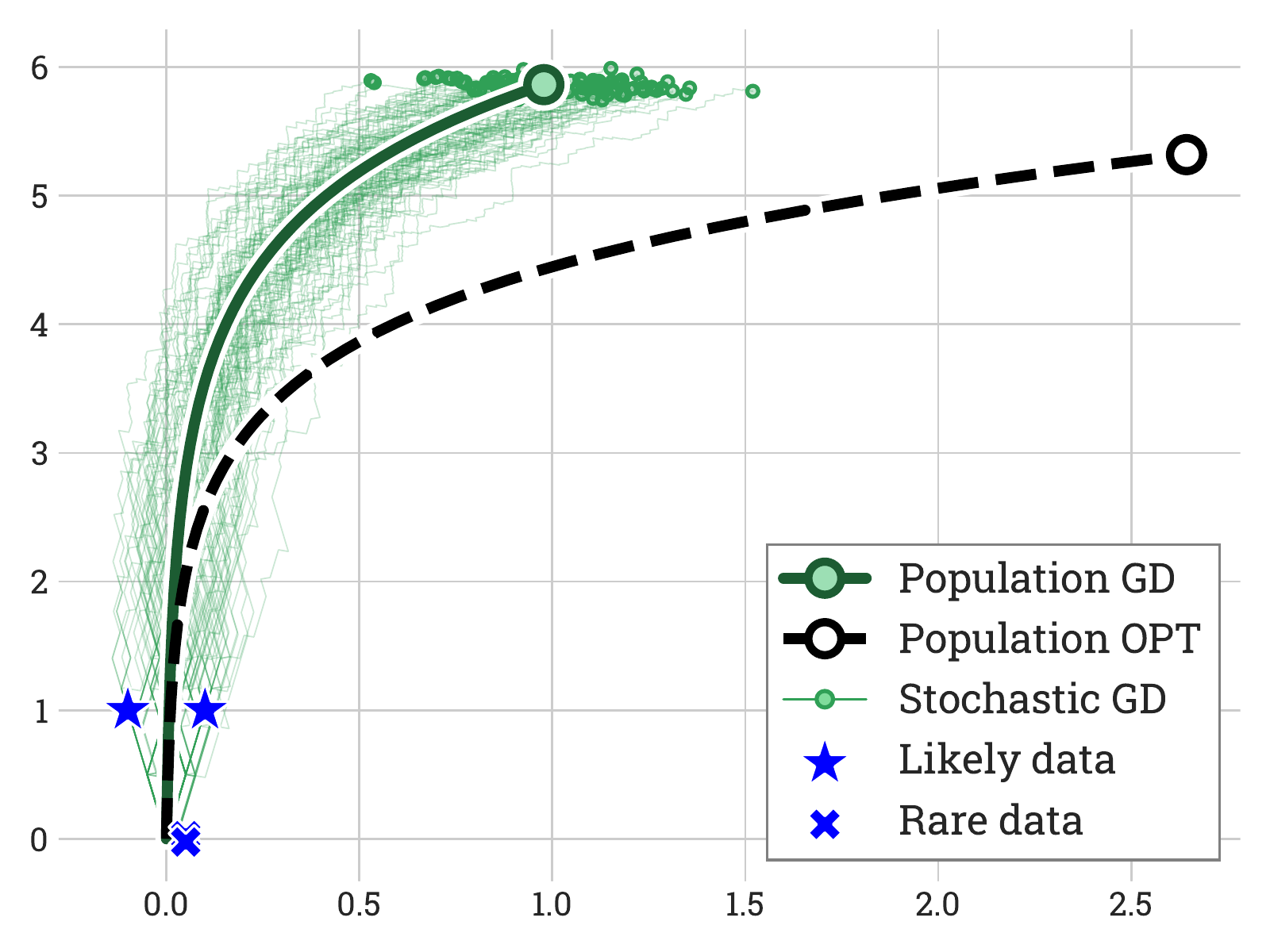}
    \caption{Logistic loss with step size $\eta = 1$.}
    \label{fig:log}
  \end{subfigure}
\caption{
    100 parallel runs of 400 stochastic GD iterations on the same data distribution in $\R^2$
    using the squared loss (cf. \Cref{fig:lsq}) and the logistic loss
    (cf. \Cref{fig:log}).
    The distribution consists of two ``likely'' upper data points (sampled with probability $90\%$),
    and two ``rare'' lower data points (sampled with probability $10\%$),
    all with common label $+1$.
    Three types of trajectory are depicted: ``population OPT'' is the curve of optimal constrained solutions
    $\{ \argmin_{\|w\|\leq B} \cR(w) : B \geq 0\}$,
    ``population GD'' is GD applied to the population risk $\cR$,
    and ``stochastic GD'' uses a fresh sample for each update.
    In all cases, the methods procrastinate convergence towards their asymptotic destinations,
    respectively the minimum norm and maximum margin solutions,
    and instead spend a good deal of time heading upwards, specifically towards low norm, low risk solutions.
    Analyzing this early trend is a goal of the present work, which is achieved through appropriate
    choices of $\wref$, as detailed in the illustrative examples in \Cref{sec:realizable,sec:markov}.
  }
  \label{fig:intro}
\end{tcolorbox}
\end{figure}

\begin{enumerate}[font=\bfseries]
  \item
    \textbf{Single coupling-based proof technique.}
    The core contribution is a single proof technique which can handle MD (which generalizes GD) and TD,
    obtaining excess risk rates with high probability without projections, constraints, regularization, or equivalents.
    Despite the extensive history and development of these methods throughout machine learning and optimization
    \citep{robbins_monro,nemirovsky_yudin,bottou2010large,kingma2014adam},
    prior work either requires projections, regularization, and constraints
    \citep{RSS12,pmlr-v99-harvey19a},
    or makes noise and comparator assumptions which effectively necessitate bounded iterates
    \citep{li2020high},
    or it provides bounds only in expectation
    \citep{hardt2016train},
    or is tailored to specific data and loss settings, for instance exponentially-tailed losses
    and linearly separable data
    \citep{nati_logistic,min_norm,shamir2021gradient},
    to mention a few.
    By contrast, the present work not only handles all such cases, it does so with an elementary and unified coupling-based
    proof technique for any \emph{quadratically-bounded loss}, or more generally fixed point mappings with
    quadratic growth,
    such as the TD update, which has no prior high probability analysis (even with projections).
    This lack of projections is relevant in contemporary usage, since deep learning
    typically has minimal or nonexistent regularization
    \citep{nati_implicit_gen,rethinking}.

  \item
    \textbf{Locally-adapted rates.}
    Even if there is some natural constraint or regularization in effect within the optimization procedure,
    the optimal solution may be unsatisfactory: e.g., it may simply be very large, and competing with it
    could require a large number of samples.
    On the other hand, the present proofs and rates only rely upon properties of reasonable reference solutions,
    which may fail to be optimal, but instead have much lower norm.

    As an illustration, consider \Cref{fig:intro}, which runs stochastic GD
    on the same (discrete) data distribution with either
    the squared loss $(y,\yhat) \mapsto (y-\yhat)^2/2$
    (cf. \Cref{fig:lsq}),
    or the logistic loss $(y,\yhat) \mapsto \ln(1+\exp(-y\yhat))$
    (cf. \Cref{fig:log}).
    Many trajectories of stochastic GD are plotted along with
    GD run directly on the population risk $\cR$ (labeled ``population GD''),
    as well as the curve of optimal constrained solutions $\{ \argmin_{\|w\|\leq B} \cR(w) : B \geq 0\}$ 
    (labeled ``population OPT'').
    All trajectories take a long time to rotate towards their asymptotic targets (respectively
    the minimum norm and max margin solutions); their early behavior is better characterized by rather 
    different low norm
    but higher risk comparators.
\end{enumerate}

In detail, the concrete contributions and organization of this work are as follows.

\begin{enumerate}[font=\bfseries]
  \item
    \textbf{\Cref{fact:md:realizable}: stochastic MD with realizable IID data.}
    This first guarantee is for stochastic MD (mirror descent with stochastic gradients, as
    detailed in \Cref{sec:notation}) on IID \emph{realizable} data;
    as discussed in \Cref{sec:notation} and \Cref{sec:realizable},
    realizability is encoded as the existence of $\wref$
    with \emph{population} risk roughly $\cO(1/t)$.
    \Cref{sec:realizable} will provide the guarantee,
    discuss the realizability condition with a few illustrative examples,
    and outline
    the general coupling-based proof technique used throughout this work.
    This realizable setting
    not only generalizes margin-based analyses for exponentially-tailed losses
    on linearly-separable data \citep{min_norm,shamir2021gradient}, it extends them to control
    the convex risk and not just misclassification risk, to settings with only approximate
    linear separability, and lastly
    handles realizable \emph{regression} settings, where $1/t$ high probability rates seem missing
    in the literature (even with projection).

  \item
    \textbf{\Cref{fact:md:general} and \Cref{fact:td}:
    stochastic MD and stochastic TD on Markovian data.}
    Dropping the realizability assumption, \Cref{sec:markov} analyzes stochastic MD
    with not just IID but Markovian data, and uses the same proof technique to handle
    the standard TD approximate fixed point method used extensively in reinforcement learning
    \citep{sutton_td}.
    For MD, the closest prior work used projections \citep{duchi2012ergodic}.
    For TD, there appear to be no prior high probability bounds, even with projections;
    the closest prior projection-free analysis is in expectation only \citep{mjt_ac},
    and most prior works make use of not only projections, but also full rank
    and mixing assumptions \citep{russo_td}.
    For MD, the squared loss is considered as an illustrative example, where stochastic GD
    is shown to adapt to local structure in the strong sense of competing with
    singular value thresholding.

  \item
    \textbf{\Cref{fact:md:heavy}, \Cref{fact:md:batch}, and \Cref{fact:mf:batch}:
    heavy-tailed and batch data.}
    As a brief auxiliary investigation to demonstrate the proof technique, rates are given in \Cref{sec:other}
    for
    MD on batch and
    heavy-tailed data.  Once again, prior work either requires projections,
    or violates one of the other goals (e.g., being custom-tailored to exponential-tailed losses
    \citep{zhang_yu_boosting,mjt_margins,nati_logistic}).
    The batch analysis includes not just discrete-time MD, but also a \emph{continuous-time
    mirror flow (MF)} analysis.

\end{enumerate}

Rounding out the organization, this introduction concludes with notation and setting in \Cref{sec:notation},
and the work itself concludes with further
related work and open problems in \Cref{sec:open}.

\subsection{Notation}
\label{sec:notation}

\paragraph{Loss functions.}
In order to handle regression and classification settings simultaneously,
each loss $\ell:\R\times\R\to\R_{\geq 0}$ will have an auxiliary scalar function $\tell$ with exactly one of two forms:
either $\ell$ is a \emph{classification (margin) loss} $\ell(y,\yhat) = \tell(\sgn(y)\yhat)$,
where $\sgn(y) := 2\1[y \geq 0] - 1 \in \{-1,+1\}$,
or $\ell$ is a \emph{regression (distance) loss} $\ell(y,\yhat) = \tell(y-\yhat)$.
Subgradients of $\ell$ will always be in the second argument,
and always exist since $\tell$ is always convex in this work;
for convenience, $\ell'$ will denote some fixed selection from $\partial \ell$.
The core loss property, \emph{quadratic boundedness},
is defined as follows.

\begin{assumption}\label{ass:qb}
  A loss $\ell$ is \emph{$(C_1,C_2)$-quadratically-bounded} (for nonnegative $C_1,C_2$)
  if
  \[
    |\ell'(y,\yhat)| \leq C_1 + C_2 \del{ |y| + |\yhat| },
    \qquad \forall y,\yhat.
    \qedhere
  \]
\end{assumption}

This property is quite pessimistic in the sense that stronger variants can be satisfied for all standard losses.
Even so, its worst-case nature demonstrates the utility of the core proof technique (which doesn't explode even
for this formulation), and captures standard losses via the following lemma.
(Throughout this work, $\|\partial f(w)\| := \sup \{ \|g\| : g \in \partial f(w)\}$.)

\begin{lemma}\label[lemma]{fact:qb:lip_or_smooth}
  If $\ell$ is $\alpha$-Lipschitz (i.e., $\sup_{z} |\partial \tilde \ell(z)| \leq \alpha$),
  then $\ell$ is $(\alpha, 0)$-quadratically-bounded.
  If $\ell$ is $\beta$-smooth (i.e., $| \tell'(z) - \tell'(\hat z)| \leq \beta |z-\hat z|$ $\forall z,\hat z$),
  then $\ell$ is $(|\partial\tell(0)|, \beta)$-quadratically-bounded.
\end{lemma}

A second crucial property is \emph{self-boundedness}, which is used in the realizable rates of \Cref{sec:realizable}.

\begin{definition}\label[definition]{defn:self-bounding}
  A loss function $\ell$ is \emph{$\rho$-self-bounding}
  if $\tell$ satisfies $\tell'(z)^2 \leq 2\rho\tell(z)$
  for all $z\in\R$.
\end{definition}

Notably, the two primary losses in machine learning, the logistic and squared losses, are both
self-bounding and quadratically-bounded.

\begin{lemma}\label[lemma]{fact:self-bounding}
  The squared loss $\ell(y,\yhat):= \frac 1 2 (y - \yhat)^2$ is $1$-smooth,
  $1$-self-bounding, and $(0,1)$-quadratically-bounded,
  whereas the
  logistic loss $\ell(y, \yhat) :=  \ln(1+\exp(-y\yhat))$ is $(1/4)$-smooth, $1$-Lipschitz,
  $(1/2)$-self-bounding, and $(1,0)$-quadratically-bounded.
\end{lemma}

A few remarks on self-bounding are in order.  Firstly, the definition has appeared before
\citep{tong_sgd}, however in a generalized form and with a calculation for the logistic loss
which implies it is $0$-self-bounding under the present definition; that the logistic loss
is $1$-self-bounding was first observed in \citep{mjt_margins},
and is crucial for obtaining $1/t$ rates under realizability.
Secondly, it may seem that self-bounding is simply a reformulation of smoothness,
but firstly it is satisfied for certain nonsmooth losses (in the sense of bounded second derivatives),
such as the exponential loss, and secondly replacing self-boundedness with smoothness breaks
the current proofs.

\paragraph{Probabilities, expectations, Markov chains, and risks.}
When data arrives IID, then $\Ex$ and $\Pr$ will respectively denote expectations
and probabilities.
Correspondingly, the risk $\cR(w)$ is defined by $\Ex_{x,y} = \ell(y, w^\T x)$.
Whenever data $((x_i,y_i))_{i=1}^t$ and a loss $\ell$ are available,
define $\ell_i(v) := \ell(y_i, x_i^\T v)$, though $\ell_{x,y}(v) := \ell(y, x^\T v)$ is also used,
whereby $\cR(w) = \Ex_{x,y} \ell_{x,y}(w)$.
Lastly, define \emph{excess risk}
\[
  \cE(w) := \cR(w) - \inf_{v\in\dom(\psi)} \cR(v),
  \qquad
  \textup{where }
  \dom(\psi) = \cbr{ u : \psi(u) < \infty };
\]
the concept and notation $\dom(\psi)$ may seem surprising, but note that all MD iterates
must satisfy $\psi(w_i) < \infty$ by definition of MD in \cref{eq:md:1}; this point will be
revisited below in the expanded discussion of MD.

With Markov chains and stochastic processes, $\Ex_{\leq i}$ will
be used to condition on $\cF_{\leq i}$,
the $\sigma$-algebra of all information up through time $i$.
It will not be necessary for the stationary processes here to be exactly Markovian
(or possess a precise stationary distribution);
instead, inspired by the \emph{Ergodic Mirror Descent} analysis
by \citet{duchi2012ergodic},
the stationarity assumption here will be approximate.

\begin{definition}\label[definition]{defn:witness}
  Let $(x_i)_{i\geq 0}$ be samples from a stochastic process,
  and let $P_i^t$ denote the conditional distribution of $x_t$
  conditioned on time $i < t$.
  For any $\epsilon\geq 0$, a triple $(\pi, \tau, \eps)$ is an \emph{approximate stationarity witness}
  if
  \[
    \sup_{t \in \Z_{\geq 0}}
    \sup_{\cF_{\leq t}}
    \textsc{tv}(P_t^{t+\tau}, \pi) \leq \eps.
  \]
  (For a similar condition in prior work, see \citep[Assumption C]{duchi2012ergodic}.)
\end{definition}

Note that if $(x_i)_{i\geq 0}$ are IID, then we can choose $(\pi,\tau,\eps) = (P_1^1, 1, 0)$,
and the corresponding stochastic MD bounds in \Cref{sec:markov} exhibit no degradation in the IID case.
For a broad variety of Markov chains,
for any $\eps>0$ we can establish
$\tau = \cO(\ln(1/\eps))$,
with hidden constants uniform in $\eps$
\citep{meyn2012markov}.
We will always bake in $\eps = 1/\sqrt{t}$, which
suggests $\tau = \cO(\ln(t))$.

For risk minimization over Markov data as in \Cref{fact:md:general},
the risk will refer to $\cR(w) := \Ex_{x,y\sim\pi} \ell_{x,y}(w)$,
where $\pi$ is an approximate stationary distribution provided by \Cref{defn:witness}.

\paragraph{Mirror descent.}
Mirror descent is a powerful generalization of gradient descent, which operates as follows.
Given a differentiable $1$-strongly-convex \emph{mirror map} $\psi$ and the corresponding
Bregman divergence $D_\psi(w,v) := \psi(w) - \sbr{\psi(v) + \ip{\nabla \psi(v)}{w-v}}$,
and a sequence of objective functions $(f_i)_{i\leq t}$, mirror descent chooses a new iterate
$w_{i+1}$ from the old iterate $w_i$ and a step size $\eta\geq 0$ and subgradient $g_{i+1}\in\partial f_{i+1}(w_i)$ and a closed convex constraint set $S$ (with $S=\R^d$ allowed)
via \cref{eq:td:1}, repeated here verbatim for convenience as
\[
  w_{i+1} := \argmin_{w\in S} \del{ \ip{\eta g_{i+1}}{w} + D_\psi(w,w_i)}.
\]
Returning to the point above that $\psi(w_{i+1}) < \infty$, note that $w_i$ is always feasible
(with finite value) for the preceding minimization problem, whereas a point $w$ with $\psi(w)=\infty$
would also have $\ip{g_{i+1}}{w} + D_\psi(w,w_i) = \infty$, and thus would never be selected.
(We assume $\psi(w_0) < \infty$ throughout.)
In the present work, typically $f_{i+1} = \ell_{i+1}$, but in \Cref{sec:md:batch} it will be
the full batch empirical risk.
As before, we will use the notation $\|\partial f_{i+1}(w)\|_* = \sup\{\|g\|_* : g\in \partial f_{i+1}(w)\}$,
and the particular choice of subgradient will never matter.
The vector space for iterates will be $\R^d$ mainly for sake of presentation,
however none of the bounds have any dependence on dimension,
and a future version may simply use a separable Hilbert space.
Norms without any subscript are simply general norms (i.e., not necessarily Euclidean).
Batch and stochastic gradient descent can be written as mirror descent via
$\Psi(v) := \|v\|^2_2/2$; for more information on mirror descent, there are many
excellent texts \citep{duchi2012ergodic,bubeck,nemirovsky_yudin}.

TD is only used in \Cref{fact:td}, and its presentation is deferred to \Cref{sec:td}.
Similarly, the mirror flow is only used in \Cref{fact:mf:batch}, and presented in
\Cref{sec:md:batch}.

\paragraph{The comparator $\wref$.}
The bounds will rely not on global minimizers, but rather on merely good comparators $\wref$.
These comparators will either be in a more general non-realizable case
(used in \Cref{fact:md:general}, \Cref{fact:md:heavy}, \Cref{fact:md:batch},
\Cref{fact:mf:batch}),
discussed momentarily, or a realizable case (used in \Cref{fact:md:realizable}),
discussed at the end of this section.  Further detailed discussion will come as illustrative
examples in \Cref{sec:realizable} and \Cref{sec:markov}.

The general assumption is that $\wref$ and $t$ satisfy $\cE(\wref) \leq \frac 1 {\sqrt t} D_\psi(\wref,w_0)$.
To make sense of the tradeoff between $t$ and $\wref$,
consider two examples of obtaining one from the other:
first consider fixing $\wref$ and identifying permissible choices of $t$,
and secondly fix $t$ and pick a single meaningful comparator.

Proceeding with the first construction, let $\wref$ be given,
and define a maximum permissible time $\tref$ as
\[
  \tref := \begin{cases}
    \infty,
    & \textup{if }\cE(\wref) = 0,
    \\
    \sbr{\frac {D_\psi(\wref,w_0)}{\cE(\wref)}}^2,
    &
    \textup{otherwise}. \end{cases}
\]
By construction, we always have $\cE(\wref) \leq D_\psi(\wref,w_0)/\sqrt{t}$
for all $t\in [0,\tref]$ (since the right hand side is decreasing in $t$).
As a quick sanity check, if $\cE(\wref) = 0$, then $\tref = \infty$, and thus
we can use global minimizers (when they exist) as comparators for all $t$.
Otherwise, when $\cE(\wref) > 0$,
we can vaguely interpret $\tref$ as saying that large comparators allow
larger choices of $t$, but for a nicer characterization of the tradeoff,
it seems more instructive
to fix $t$ and choose a good $\wref$, as follows.

Utilizing a slightly more flexible definition, let $\lambda > 0$ be arbitrary, and define
an optimal regularized solution $\uref(\lambda)$ as
\[
  \uref(\lambda) := \argmin\cbr{ \cR(u) + \frac \lambda 2 D_\psi(u,w_0) : u\in \R^d }.
\]
The following \namecref{fact:uref} shows that we can \emph{always} use this comparator.

\begin{proposition}\label{fact:uref}
  Given any $\lambda > 0$,
  then $\displaystyle \cE(\uref(\lambda)) \leq \lambda D_\psi(\uref(\lambda), w_0)$.
\end{proposition}

In particular, given $t$, we can always use $\wref := \uref(1/\sqrt{t})$, whereby
\Cref{fact:uref} implies
$\cE(\wref)\leq D_\psi(\wref)/\sqrt{t}$.
This is perhaps much more
interpretable than the opaque comparator condition itself, however it is weaker: e.g., returning
to \Cref{fig:intro}, the general condition allows comparison to all given curves,
and not just a single curve of regularized solutions
(namely $\{\uref(\lambda) : \lambda \geq 0\}$).
More notably, the proof of \Cref{fact:uref} uses a fairly hideous weakening of the inequality.

Lastly, to discuss the realizable condition $\cR(\wref) \leq D_\psi(\wref, w_0)/t$,
suppose for simplicity
that $\cR(w) = \cE(w)$ for all $w$ (that is, $\inf_{v\in\dom(\psi)}\cR(v) = 0$).
Then the condition becomes $\cE(\wref) \leq D_\psi(\wref, w_0)/t$, and we may
once again apply \Cref{fact:uref}: given $t$, we can always choose $\wref:=\uref(1/t)$.

\section{Realizable case, illustrative examples, and proof scheme}
\label{sec:realizable}

The first bound is for self- and quadratically-bounded losses,
and requires \emph{realizability}: as discussed at the end of \Cref{sec:notation},
this corresponds to the existence of $\wref$ with $\cR(\wref) \leq D_\psi(\wref,w_0)/t$,
which will be discussed momentarily for the logistic loss.

\begin{theorem}\label[theorem]{fact:md:realizable}
  Suppose $\ell$ is convex, $(C_1,C_2)$-quadratically-bounded, and $\rho$-self-bounding.
  Let $t$ be given, and suppose
  $((x_i,y_i))_{i\leq t}$ are IID samples with
  $\max\{\|x_i\|_*,|y_i|\}\leq 1$ almost surely.
  Let reference solution $\wref$ and initial point $w_0$ be given,
  and suppose $\wref$ satisfies $\cR(\wref) \leq \rho D_\psi(\wref,w_0) / t$,
  and let $C_4$ be given so that $\max_{j<t}|\ell_{j+1}(\wref)| \leq C_4$ almost surely.
  Then with probability at least $1-2t\delta$, every $i\leq t$ satisfies
  \begin{align*}
    \frac {8}{3i\eta} D_\psi(\wref, w_{i})
    + \frac 1 i \sum_{j<i} \cR(w_j)
    \leq
\frac {2 B^2}{i \eta}
    + \frac {4}{\eta} \cR(\wref),
  \end{align*}
where $B := \sqrt{1 + C_1 + C_2(1+   \|\wref\|) + C_4}
  \max\cbr{ 1, 4\sqrt{ D_\psi(\wref,w_0)},\sqrt{\frac{64C_4}{\rho} \ln\frac 1 \delta}}$
  and $\eta \leq \frac{1}{2\rho}$.
\end{theorem}

All bounds in this work will have roughly the form of \Cref{fact:md:realizable},
which can be summarized as follows.
As stated in the introduction, the bound has a regret-style average risk on the left hand side, and
the risk of the comparator $\wref$ in the right hand side.
Unusual elements are the control for all times $i<t$,
the left hand side term $D_\psi(\wref, w_{i})$, and the coefficient exceeding one on $\cR(\wref)$.
The control for all times $i<t$ is an artifact of the proof scheme,
and will be discussed below.  The left hand side term $D_\psi(\wref,w_i)$ is crucial to the operation of the proof;
it is an \emph{implicit bias} which prevents iterates from growing too large.
This term is dropped in all standard presentations of mirror descent
\citep{bubeck,duchi2012ergodic,nemirovsky_yudin}, but was exploited in the original perceptron convergence proof
\citep{novikoff}.
The large coefficient on $\cR(\wref)$ is a consequence of realizability, and can not be removed with the current proof scheme.

\subsection{Illustrative example: the logistic loss and approximate separability}
\label{sec:homotopy}

\begin{figure}[t]
 \begin{tcolorbox}[enhanced jigsaw, empty, sharp corners, colback=white,borderline north = {1pt}{0pt}{black},borderline south = {1pt}{0pt}{black},left=0pt,right=0pt,boxsep=0pt,rightrule=0pt,leftrule=0pt]
   \centering
\begin{subfigure}[t]{0.48\textwidth}
    \centering
    \includegraphics[width = \textwidth]{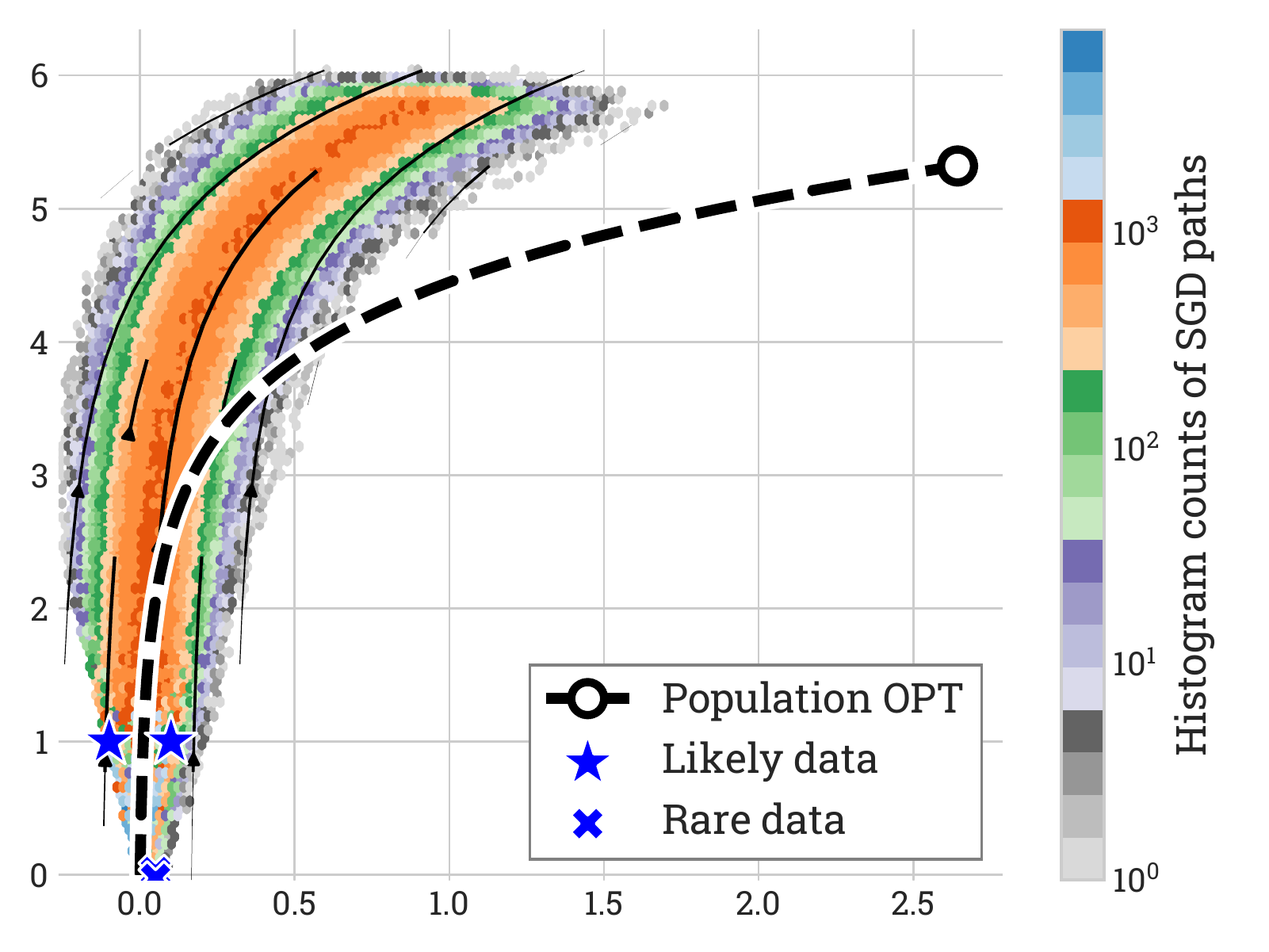}
    \caption{Another visualization of \Cref{fig:log}.}
    \label{fig:log:hexbin}
  \end{subfigure}
  \hfill
  \begin{subfigure}[t]{0.48\textwidth}
    \centering
    \includegraphics[width = \textwidth]{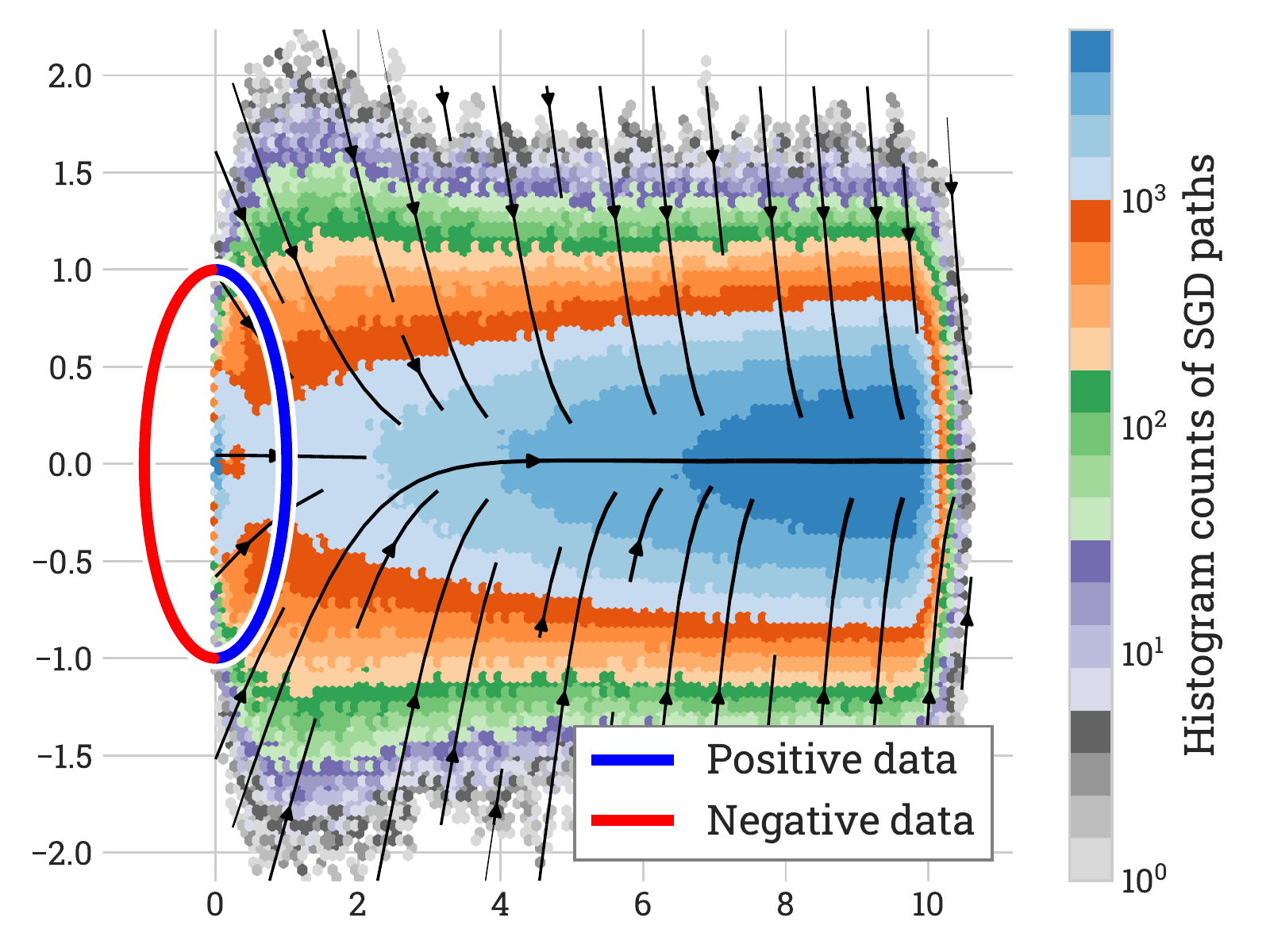}
    \caption{Spherical data from \Cref{fact:margin:zero}.}
    \label{fig:circle:log}
  \end{subfigure}
  \caption{
    These plots show SGD on the logistic loss, but rather than depicting many trajectories as a filigree
    (as in \Cref{fig:intro}),
    they are histogrammed into hexagonal bins, with additional black lines showing the
    gradient vector field.
    The color scheme gradations are logarithmic, and seem to exhibit an exponential concentration around the GD path,
    a phenomenon not established in this work.
    \Cref{fig:log:hexbin} shows this visualization technique on the same data from
    \Cref{fig:intro}, whereas \Cref{fig:circle:log} shows the spherical $0$-margin data from
    \Cref{fact:margin:zero}.
}
  \label{fig:intro:hexbin}
\end{tcolorbox}
\end{figure}

To investigate \Cref{fact:md:realizable} more closely, consider the logistic loss $\ell(y,\yhat) = \ln(1+\exp(-y\yhat))$
on classification data.
In typical implicit regularization works, the logistic loss is treated as having an
exponential tail, and inducing convergence to maximum margin directions
\citep{zhang_yu_boosting,mjt_margins,nati_logistic,shamir2021gradient}.
As in \Cref{fig:log}, it can take a while for the maximum margin asymptotics to kick in;
for example, the risk rate for SGD in \citep[Theorem 1.1]{min_norm} is $\cO(\ln(t) / (t\gamma^2))$, where
the population margin $\gamma$ can be taken to mean $\Pr[\bar u ^\T x y \geq \gamma] = 1$ for some unit vector $\bar u$.
Returning to \Cref{fig:log}, the global margin $\gamma$ is very small, but the initial dynamics are governed by
the much larger margin of the likely points oriented vertically.
As a first step towards formalizing this and connecting back to \Cref{fact:md:realizable},
consider the following \emph{approximate realizability/separability} characterization,
and its consequences on the choice of $\wref$.

\begin{proposition}\label[proposition]{fact:md:margin}
  Let $t$ be given.
  Suppose $\ell(y,\yhat) = \ln(1+\exp(-\sgn(y)\yhat))$ is the logistic loss,
  the data satisfies $\|x\|_2\leq 1$ and $y\in\{\pm 1\}$ almost surely,
  and that there exists a unit vector $u_t\in\R^d$ and a scalar $\gamma_t > 0$
  so that
  $\Pr[u_t^\T xy \geq \gamma_t] \geq 1-1/t$.
  Then the reference solution $\wref := u_t\ln(t) / \gamma_t$ satisfies
  $\cR(\wref) \leq \del{2 + \ln(t)/\gamma_t}/t$.
\end{proposition}

Applying \Cref{fact:md:realizable} with this $\wref$ and with step size $\eta=1$ gives, with probability
at least $1-\delta$,
\[
  \frac{8}{3t} \|w_t - \wref\|^2 + \frac {1}{t} \sum_{i<t} \cR(w_i)
  \leq
  \cO\del{\frac {\ln(t)^2\sqrt{\ln(t/\delta)}}{t\gamma_t^2}}.
\]
As a sanity check, if the global maximum margin $\gamma \leq \gamma_t$ is used,
then this bound matches the bound from prior work mentioned before
\citep[Theorem 1.1]{min_norm},
and also matches standard margin-based generalization bounds \citep{boosting_margin}.
A key difference is that $\gamma_t$ is not a hard margin, but allows some fraction of margin
violations.  In particular, returning to \Cref{fig:log},
if we choose $t = 10$, then we can choose $\gamma_t$ to be the large ``likely data'' margin,
and our convergence rate scales with this $1/\gamma_t^2$, rather than the much smaller ``rare data'' margin.

Rather than relying on an opaque $\gamma_t$, here is another example where $\gamma_t$ may be calculated.
Consider \Cref{fig:circle:log}, where there is a perfect but $0$-margin classifier (thus breaking standard bounds),
and the marginal distribution of $x$ along this perfect classifier is uniform.
In this setting, the optimal predictor of length $r$ is unique and achieves risk only $1/r$,
in contrast to the standard margin setting (roughly as in \Cref{fact:md:margin},
where one hopes for a predictor of length $\ln(r)/\gamma$ for risk $1/r$).

\begin{proposition}\label[proposition]{fact:margin:zero}
  Let dimension $d\geq 2$ be given,
  let $\mu_0$ denote the uniform probability density on the sphere
  $\cS_{d-1}:=\{x \in \R^d : \|x\|=1\}$,
  and let $\mu$ denote a reweighting of $\mu_0$ along the axis $e_1$ so that every orthogonal
  slice has equal density, meaning $\dif\mu(x) = p(x_1)\dif\mu_0(x)$
  for some $p$, whereby $\Pr_\mu[\{ x\in \cS_{d-1} : -1 \leq a \leq x_1 \leq b \leq 1\}]= (b-a)/2$,
  and suppose $\Pr[ y = 1 | x] = \1[x_1 \geq 0]$.
  Then for any norm $r>0$, the vector $u_r := re_1$ is the unique minimizer of $\cR$ with norm $r$,
  and moreover if $r\geq 1$ then
  \[
    \envert{ \cR(u_r) - \frac {\pi^2}{12r} } \leq \frac {2}{\exp(r)}.
  \]
\end{proposition}

Now consider applying the (approximately) realizable analysis in \Cref{fact:md:realizable}
to this setting.  Choosing $\eta = 1$ as suggested there,
and 
$\wref = e_1 / t^{1/3}$
as suggested by \Cref{fact:margin:zero} after optimizing terms,
then
$\max\{\|\wref\|, \cR(\wref)\} = \cO(1/t^{1/3})$
and $\cR(\wref) = \cO(\|\wref\|^2 / t)$ as required by the realizability conditions
in \Cref{fact:md:realizable}.
Thus, with probability at least $1-\delta$,
\[
  \frac 8 {3 t} \|\wref - w_t\|^2
  +
  \frac 1 t \sum_{i<t} \cR(w_i)
  \leq
  \cO\del{\frac {\ln(t/\delta)}{t^{1/3}}}.
\]
There does not appear to be any prior work analyzing these infinitessimally-separable
scenarios;
furthermore, such examples necessitated the realizability formulation in
\Cref{fact:md:realizable}.

\subsection{Proof sketch: the core coupling-based argument}\label{sec:proof_sketch}

This subsection provides the basic form of the coupling-based argument used within
all proofs in this work.  For sake of presentation,
the argument is given for the simpler setting of
stochastic gradient descent with IID data and a step size $\eta = \cO(1/\sqrt{t})$,
and a few remarks at the end show how to adjust it to obtain a $\cO(1/t)$ rate
as in \Cref{fact:md:realizable}.

The proof scheme consists of the following three steps.

\begin{enumerate}[font=\bfseries]
  \item
    \textbf{Coupling unconstrained iterates $(w_i)_{i<t}$ with constrained iterates $(v_i)_{i<t}$.}
    Because $(w_i)_{i<t}$ are unconstrained, it is unclear how to apply standard concentration inequalities
    to them.  
    Instead, define \emph{projected} iterates $(v_i)_{i<t}$ which are coupled to $(w_i)_{i<t}$ in the following
    strong sense: $v_0 = w_0$, and thereafter, $v_{i+1}$ is defined using the \emph{same} randomness
    as $w_{i+1}$, meaning
    \begin{align*}
      w_{i+1} 
      &:= w_i - \eta \partial_w \ell(y_{i+1}, x_{i+1}^\T w_i),
      \\
      v_{i+1} 
      &:=\Pi_S\del{ v_i - \eta \partial_v \ell(y_{i+1}, x_{i+1}^\T v_i) },
    \end{align*}
    where the constraint set $S := \cbr{ v \in \R^d : \|v-\wref\|\leq B_w}$ has a few key choices.
    Firstly, it projects onto a ball around the desired comparator $\wref$; algorithmically, this
    would require clairvoyantly re-running the algorithm with knowledge of $\wref$, but here it is
    only used as a mathematical construct.  Since $(v_i)_{i\leq t}$ explicitly depends on $\wref$,
    relating $w_i$ to $v_i$ will in turn relate $w_i$ to $\wref$.
    A description of the radius $B_w$ will come shortly.

  \item
    \textbf{Implicitly-biased MD analysis of $(v_i)_{i\leq t}$.}
    Because $(v_i)_{i\leq t}$ are constrained to a small ball around $\wref$,
    we can easily apply MD and concentration guarantees and expect all quantities to scale
    with properties of $\wref$.
    Concretely, following the standard MD proof scheme specialized to GD via $\Psi(w) = \frac  1 2 \|w\|_2^2$
    (whereby $D_\psi(\wref,w) = \frac 1 2 \|\wref - w\|_2^2$),
    and writing $h_{j+1} := \partial_v \ell(y_{j+1}, x_{j+1}^\T v_j)$ for the stochastic gradient at time $j+1$
    for $v_j$,
    \begin{align*}
      \|v_{j+1} - \wref\|^2_2
      \leq
      \|v_{j} - \wref\|^2_2 + 2\eta \sbr{\ell_{j+1}(\wref) - \ell_j(v_{j})} + \eta^2 \|h_{j+1}\|^2,
    \end{align*}
    which after recursing and rearranging (alternatively applying $\sum_{j<i}$ to both sides) gives
    \begin{align*}
      \|v_i - \wref\|^2_2
      \leq
      \|v_{0} - \wref\|^2_2 + 2\eta \sum_{j<i} \sbr{
      \ell_{j+1}(\wref) - \ell_{j+1}(v_{j}) + \eta^2 \|h_{j+1}\|^2}.
    \end{align*}
    Applying Azuma's inequality,
    with probability at least $1-\delta$,
    \begin{align}
      \|v_i - \wref\|^2_2
      &\leq
      \|v_{0} - \wref\|^2_2 + 2\eta \sum_{j<i} \sbr{
        \cR(\wref) - \cR(v_j)
      }
+
      \eta \Big[\textup{deviations} + \eta \sum_{j<i} \|h_{j+1}\|^2 \Big].
      \label{eq:proof_sketch:qb}
    \end{align}
    What is the magnitude of the final bracketed error term?  Azuma's inequality scales
    with the range of the relevant random variables,
    and thus if the loss has quadratic growth, we can expect the entire second line
    to be $\cO(B_w^2)$, which deserves quite a bit more discussion.

    This quantity $\cO(B_w^2)$ (and the choice of quadratically-bounded losses) is crucial.
    The left hand side of the bound has $\|v_i - \wref\|^2_2$, which is at most $B_w^2$ by the choice of $S$.
    As such, if the $\cO(B_w^2)$ in the right hand side has a leading constant less than $1$,
    then \emph{the projection operation is never invoked}, and we should be able to show $w_i = v_i$.
    In fact, this observation was the starting point of this work, and ``quadratically-bounded loss'' is merely
    a reverse-engineered concept to make it go through.
    Moreover, it is clear that none of this would be possible if the left hand term
    $\|v_i-\wref\|^2_2$ were deleted, as is standard in MD.

  \item
    \textbf{Proving $(w_i)_{i\leq t} = (v_i)_{i \leq t}$ via induction.}
    We are now in position to complete the proof.

    Let $E$ denote the failure event for the earlier regret guarantee in \cref{eq:proof_sketch:qb},
    which rules out certain wild trajectories for $(v_i)_{i\leq t}$.
    The underlying sample space for this event is $((x_i,y_i))_{i\leq t}$,
    and therefore this event also controls the behavior of $(w_i)_{i\leq t}$;
    in fact, this proof will show
    that ruling out $E$ deletes not just the wild trajectories of $(v_i)_{i \leq t}$, but also that one
    may interpret these projected iterates
    as mere proxies to get a handle on the wild trajectories of $(w_i)_{i\leq t}$.
    This proof technique is then a truncation argument, as is standard throughout probability theory.

    In detail, consider $w_{i+1}$, and suppose the inductive hypothesis $(w_j)_{j\leq i} = (v_j)_{j\leq i}$.
    Writing out a \emph{deterministic} gradient descent inequality for $w_{i+1}$ (cf. \Cref{fact:md})
    and then invoking the inductive
    hypothesis to transplant $(v_j)_{j\leq i}$, gives (under the complement
    of failure event $E$)
    \begin{align*}
      \|w_{i+1} - \wref\|^2_2
      &=
      \|w_{0} - \wref\|^2_2 + 2\eta \sum_{j\leq i} \sbr{
        \ell_{j+1} (\wref) - \ell_{j+1}(w_j)
      }
      +
      \eta^2
      \sum_{j\leq i} \|\partial \ell_{j+1}(w_j)\|^2
      \\
      &=
      \|v_{0} - \wref\|^2_2 + 2\eta \sum_{j\leq i} \sbr{
        \ell_{j+1} (\wref) - \ell_{j+1}(v_j)
      }
      +
      \eta^2
      \sum_{j\leq i} \|\partial \ell_{j+1}(v_j)\|^2
      \\
      &\leq
      \|v_{0} - \wref\|^2_2 + 2\eta \sum_{j\leq i} \sbr{
        \cR(\wref) - \cR(v_j)
      }
      +
      i \eta^2
      \cO(B_w^2).
    \end{align*}
    With some tuning of $\eta$ and $B_w$, the final term $i\eta^2 \cO(B_w^2)$ is in fact strictly less than $B_w^2$,
    which suffices to imply projections are never invoked, and $w_{i+1}=v_{i+1}$.
    This argument is repeated for every iteration $i\leq t$, so in fact there are $t$ different failure events
    $(E_i)_{i\leq t}$, and unioning them together gives the final statement.
\end{enumerate}

The preceding proof was for GD not MD, but the standard MD proof scheme is identical \Cref{fact:md},
even with the left hand implicit bias term added in.

Handling the realizable case has a few important differences.  The first is that the squared gradient term
$\|\partial \ell_{j+1}(v_j)\|_*^2$ is swallowed into the loss term via the definition of $\rho$-self-bounding.
Moreover, to obtain a rate $1/t$ not $1/\sqrt{t}$, Freedman's inequality is used rather than Azuma's inequality,
which needs the conditional variances to be small (which invokes realizability).
Lastly, to allow for a simple step size, two separate concentration inequalities are applied: one to control norms,
and another to control risks; using just one concentration inequality would give similar rates but require
a messy step size as in
\Cref{fact:md:general}.

\section{General MD analysis, illustrative examples, and TD analysis}
\label{sec:markov}

The section exhibits slower rates,
meaning $1/\sqrt{t}$ rather than the $1/t$ in \Cref{fact:md:realizable},
but allows an important generalization: the data need not be IID, but instead is approximately Markovian (cf. \Cref{defn:witness}),
and need not be realizable.

The first bound here is for stochastic MD.
As discussed at the end of \Cref{sec:notation},
the condition on $\wref$ is now $\cE(\wref) \leq D_\psi(\wref,w_0)/\sqrt{t}$,
and can always be satisfied by plugging in a regularized iterate $\uref(1/\sqrt{t})$ defined
there.

\begin{theorem}\label[theorem]{fact:md:general}
  Suppose $\ell$ is convex and $(C_1,C_2)$-quadratically-bounded.
  Let $t$ be given, and suppose
  $((x_i,y_i))_{i\leq t}$ are drawn from a stochastic process
  with approximate stationarity witness $(\pi,\tau,1/\sqrt{t})$ with
  $\max\{\|x_i\|_*,|y_i|\}\leq 1$ almost surely.
  Let reference solution $\wref$ and initial point $w_0$ be given,
  and suppose $\wref$ satisfies $\cE(\wref) \leq D_\psi(\wref,w_0) / \sqrt{t}$.
  Then with probability at least $1-t\tau\delta$, every $i\leq t$ satisfies
  \begin{align*}
    \frac {1}{i\eta} D_\psi(\wref, w_i)
    +
    \frac 1 i \sum_{j<i} \cR (w_j)
    &\leq
    \frac {B_w^2}{8i\eta}
    +
    \cR(\wref),
  \end{align*}
  where $B_w = \max\big\{ 1, \1[C_2 > 0] \|\wref\|, 4\sqrt{ D_\psi(\wref,w_0)}\big\}$
  and $\eta \leq \frac{1}{4096\max\{1,C_1, C_2\} \sqrt{t\tau\ln(1/\delta)}}$.
\end{theorem}

Illustrative examples of \Cref{fact:md:general} will be provided shortly in \Cref{sec:md:general:illustrative}.
The form of the bound is similar to \Cref{fact:md:realizable}, but has a rate $\cO(1/\sqrt{t})$ after expanding $\eta$.
Unlike \Cref{fact:md:realizable}, the step size is messy; this seems necessary with the current proof technique,
which seems to have no recourse but to use $\eta$ to swallow some terms; the realizable analysis was able to avoid
this thanks to applying two concentration inequalities, the first of which relied heavily on realizability.

The proof of \Cref{fact:md:general} follows the sketch in \Cref{sec:proof_sketch} exactly, with two exceptions.
The first is that GD is replaced by MD, following a nearly standard analysis with an added
non-standard implicit bias term $D_\psi(\wref,w_i)$ (cf. \Cref{fact:md}).
The second difference is that Azuma's inequality is replaced with a Markov chain concentration inequality,
which itself uses a standard technique of treating the data as $\tau$ interleaved sequences
of nearly-IID data, and applying Azuma's inequality to each.
This concentration inequality is detailed in \Cref{fact:conc:markov},
but is abstracted from a proof due to \citet{duchi2012ergodic}.

\subsection{TD analysis}
\label{sec:td}

The second Markovian guarantee is on TD.
It is not necessary to be familiar with any RL concepts to make sense of this theorem, and
in fact it can be stated as a fixed point property, but here is some brief background.
The sequence $(x_i)_{i\geq 0}$ with $x_i\in\R^d$ is interpreted as combined state/action vectors,
and instead of labels there are scalar rewards $(r_i)_{i\geq 1}$,
whose conditional distribution is fully determined by the preceding state/action vector,
meaning $r_{i+1}|\cF_{\leq i} = r_{i+1} | x_i$.
The stochastic TD update is
\begin{equation}
  w_{i+1} := w_i - \eta \td_{i+1}(w_i),
  \quad\textup{where }
  \td_{i+1}(v) = x_i \del{\ip{x_i - \gamma x_{i+1}}{v} - r_{i+1}},
  \label{eq:td:G}
\end{equation}
where the \emph{discount factor} $\gamma \in (0,1)$ is fixed throughout.

In prior work, this method is only studied in expectation, 
often with a variety of boundedness/projection 
and full rank conditions \citep{zou2019finite},
or further stationarity
and sampling conditions \citep{bhandari2019global}.  Some recent work has
aimed to reduce these assumptions, but still was only able to achieve bounds
in expectation \citep{mjt_ac}.
Meanwhile, invoking essentially the same proof as for \Cref{fact:md:general} leads to a high probability guarantee;
the only real difference is that the deterministic MD analysis (from \Cref{fact:md}) is replaced with
a similar deterministic TD analysis (from \Cref{fact:td:det}),
even though TD is not in any sense a gradient method.

\begin{theorem}\label{fact:td}
  Let a stochastic process $((x_i,r_i))_{i\geq 0}$ be given,
  where $(x_i)_i\geq 0$ form a Markov chain and $\max\{\|x_i\|,|r_i|\}\leq 1$ almost surely,
  and define auxiliary random variables
  $\zeta_{i+1} = (x_i,x_{i+1},r_{i+1})$,
  and let
  $(\pi, \tau, 1/\sqrt{t})$
  denote an
  approximate stationarity witness for $(\zeta_i)_{i\geq 1}$.
Let reference solution $\wref$ be given with $\|\Ex_{\zeta\sim\pi}\td_\zeta(\wref)\| \leq \|\wref - w_0\|^2 /\sqrt{t}$,
  where $G_\zeta(\wref) := x\del{\ip{x -\gamma x'}{\wref} - r}$ for $\zeta = (x,x',r)$.
  Then with probability at least $1-t\tau\delta$,
  \begin{align*}
    \|w_{t} - \wref\|^2
    +\eta (1-\gamma)^2 \sum_{i<t} \Ex_{x\sim \pi} \ip{x}{w_i - \wref}^2
    &\leq
    B_w^2 + \frac {t\eta B_w}{512} \enVert{ \Ex_{\zeta\sim\pi} \td_\zeta(\wref) }
    ,
  \end{align*}
  where $B_w = \max\{ 1, 4\|\wref\|, 4\|w_0 - \wref\|\}$
  and
  $\eta \leq \frac {1}{1024 \sqrt{t\tau\ln(1/\delta)}}$.
\end{theorem}

Notably, as a parallel to the approximate optimality of $\wref$ in \Cref{fact:md:realizable,fact:md:general},
the reference solution in \Cref{fact:td} need only be an \emph{approximate} fixed point:
$\|\Ex_{\zeta\sim\pi}\td_\zeta(\wref)\| = \cO(1/\sqrt{t})$.

\subsection{Illustrative examples}\label{sec:md:general:illustrative}

\paragraph{Squared loss.}
First consider the squared loss $\ell(y,\yhat) := (y-\yhat)^2/2$.
Typically $\cR$ in this setting is treated as strongly convex (perhaps along a subspace),
and SGD converges at a rate $1/t$.
Unfortunately, this rate also scales with $1 / \sigma_{\min}^2$, the inverse
of the smallest positive eigenvalue of the population covariance,
a quantity which has no reason to be small.

The goal of the present work is to eschew global dependencies, and depend on local properties;
this is also exhibited in \Cref{fig:lsq}, where both stochastic GD as well as GD on $\cR$ spend
a long time pointing \emph{away} from the minimum norm solution of $\cR$ they eventually converge to.

As a concrete construction of $\wref$, consider the case of \emph{singular value thresholding}:
rather than seeking out the population solution $\bar w := [ \Ex xx^\T ]^{+} \sbr{ \Ex xy}$, where
the ``$+$'' denotes the pseudoinverse, consider $\bar w_k := [ \Ex xx^\T ]^{+}_k \sbr{\Ex xy}$, where $k$ truncates the
spectrum of $\Ex xx^\T$ to have only the $k$ largest eigenvalues.
Correspondingly, suppose $w_0 = 0$, and choose $t$ small enough so that
\[
  \cR(\bar w_k) \leq \frac {\|\bar w_k\|^2}{2 \sqrt{t}} + \inf_v \cR(v)
  = \frac {D_\psi(\bar w_k, w_0)}{\sqrt{t}} + \cR(\bar w);
\]
whenever this holds, this $\bar w_k$ and $t$ can be plugged in to \Cref{fact:md:general},
giving a stochastic GD guarantee which not only competes with $\cR(\bar w_k)$, but moreover
the constants in the rate scale with $\|\bar w_k\|^2$, which is on the order $1 / \sigma_{k}^2$,
rather than being on the order $1/\sigma_{\min}^2$ as with the minimum norm least squares solution $\bar w$.
This gives some explanation of the behavior of \Cref{fig:lsq}: early in training, the path is closer to low norm
solutions such as the singular value thresholded solution $\bar w_k$.
Said another way, stochastic GD competes with $\bar w_k$ \emph{for every $k$} without any specialized algorithm!
Lastly, a similar connection is to the related ridge regression solutions, which is provided
by $\uref(1/\sqrt{t})$ and \Cref{fact:uref}.

\paragraph{Univariate medians.}
To close with a toy but fun example,
consider encoding the problem of finding medians as regression with a loss
$\ell(y,\yhat) := |y-\yhat|$, which is Lipschitz, nonsmooth, and unbounded;
in this setting, we will fix $x_{i+1} = 1\in\R^1$, whereby $\yhat_{i+1} := w_i$ is our estimate
at time $i+1$.
Applying stochastic GD to this problem leads to a pleasing algorithm:
given a new test point $y_{i+1}$, simply $w_{i+1} := w_i + \eta \sgn(y_{i+1} - w_i)$,
meaning we move right or left by a fixed increment $\eta$ depending on whether the new test point
is to our right or left.
Standard analyses of this method would require projections or regularization (and some outer doubling loop to guess the radius),
but \Cref{fact:md:general} can handle a direct
the projection-free stochastic GD.

\section{Final examples: batch data and heavy-tailed data}
\label{sec:other}

This final set of results will relax two conditions which may have seemed necessary to the proof scheme in \Cref{sec:proof_sketch}:
data may be heavy-tailed, and data may be handled as a batch.

\subsection{Heavy-tailed data}

All preceding sections required bounded data: $\max\{\|x\|_*,|y|\}\leq 1$ almost surely.
Instead, the following bound is similar to the non-realizable setting of \Cref{fact:md:general},
except the data is IID, and may have two types of heavy tail.

\begin{theorem}\label{fact:md:heavy}
  Suppose $\ell$ is convex and $(C_1,C_2)$-quadratically-bounded.
  Let $t$ be given, and suppose
  $((x_i,y_i))_{i\leq t}$ are drawn IID
  with auxiliary random variables
  $Z_i := \max\{1, \|x\|_*^4, |y|^4\}$
  satisfying one of the following two tail behaviors with a corresponding constant $C$.
  \begin{enumerate}[font=\bfseries]
    \item
      \textbf{Subgaussian tails.}
      Each $Z_i$ is subgaussian with variance proxy $\sigma^2$,
      and define
$
        C:= \Ex Z_1 + 2\sigma \sqrt{\ln(1/\delta)/t}
        .
        $

    \item
      \textbf{Polynomial tails.}
      Defining a moment bound
     $
M := \max\{p/e,\ {}
\sup_{2\leq r\leq p} \Ex | Z_i  - \Ex Z_i |^r \}
        $
for each $Z_i$
      for some power $p$ satisfying $8| p$,
      define
      $
        C:= 
        \Ex Z_1 + 2M\del{\frac 2\delta}^{1/p} / \sqrt{t}.
      $
  \end{enumerate}
  Let reference solution $\wref$ and initial point $w_0$ be given,
  and suppose $\wref$ satisfies $\cE(\wref) \leq D_\psi(\wref,w_0) / \sqrt{t}$.
  Then with probability at least $1-2t\delta$, every $i\leq t$ satisfies
  \begin{align*}
    \frac {1}{i\eta} D_\psi(\wref, w_i)
    +
    \frac 1 i \sum_{j<i} \cR (w_j)
    &\leq
    \frac {B_w^2}{8i\eta}
    +
    \cR(\wref),
  \end{align*}
  where $B_w = \max\big\{ 1, C_2 \|\wref\|, 4 \sqrt{ D_\psi(\wref,w_0)}\big\}$
  and $\eta \leq \frac{1}{4096 \max\{1,C_1, C_2\} \sqrt{ t (1+C) \ln(1/\delta)}}$.
\end{theorem}

A few remarks are in order.  Firstly, with polynomial tails, the bound is not what is
generally called a high probability bound, as the familiar $\ln(1/\delta)$ is replaced with
$(1/\delta)^{1/p}$.  This has a particularly bad interaction with union bounds: in fact, since
the proof technique utilized a union bound over all $t$ iterations, this term should in fact
be interpreted as $(t/\delta)^{1/p}$.  Expanding the corresponding term $C$ in the final bound,
the rate becomes
$\max\{ t^{-1/2}, t^{- 1 + 1/(2p)} \}$, which is reasonable, though the dependence on $1/\delta$
is still unpleasant.

A second remark is on the proof.  The mirror descent guarantee in the case of general data
with smooth losses ends up having terms of the form $\sum_{i<t} \max\{\|x\|_*^4, |y|^4\}$ appear
in a few places, which were simply upper bounded by $t$ in the earlier general analyses.
The step size $\eta$
is then made small to swallow these terms (i.e., the $C$ above appears within $\eta$),
but there are still two issues: firstly, $C$ must be controlled, and secondly, as this quantity
is random, we can not simply invoke Azuma's inequality with a random range.
To solve the first problem, there exist a variety of heavy tail concentration inequalities,
as detailed in the proofs in the appendix.  For the second problem, we use a very nice
variant of Azuma's inequality which allows the ranges to not be specified up front
\citep[Problem 3.11]{vanhandel}.

It appears guarantees of this type have not appeared before;
the most similar analyses consider specialized scenarios and moreover
modify the descent method,
for instance by using minibatches with specially-tuned batch sizes
to exhibit strong convexity structure assumed to hold over the population
\citep{JMLR:v23:21-0560},
or by gradient clipping and similar procedures
\citep{NEURIPS2020_abd1c782,pmlr-v125-davis20a,nazin2019algorithms}.

\subsection{Batch data}
\label{sec:md:batch}

All the bounds in this work so far have worked with stochastic data, arriving one point at a time;
correspondingly, the concentration inequalities seemed to rely upon martingale structure.
This section will study data which arrives as a single batch, and is re-used in every
iteration, with both discrete-time updates and with continuous-time updates.

To first develop the batch MD updates,
the iteration-specific loss $\ell_{j+1}$ will be replaced with (sub)gradient
of the \emph{empirical risk} $\hcR$, meaning
\begin{equation}
  g_{i+1} := \partial_w \hcR(w_i),
  \qquad
  \textup{where }
  \hcR(w_i) := \frac 1 n \sum_{k=1}^n \ell(y_k, x_k^\T w_i).
  \label{eq:batch_md}
\end{equation}
Rather than a concentration inequality, a Rademacher complexity bound will be used
(see \citep{shai_shai_book} for background).
This comes at a
cost, and will use the following additional requirement on the norm $\|\cdot\|$;
on an intuitive level, this condition gives a sort of analogue of H\"older's inequality
for Rademacher complexity, and will be discussed further after the statement.

\begin{definition}\label{defn:rad_norm}
  Say $\|\cdot\|$ is $C_6$-Rademacher if, for any scalars $B\geq 0$ and $B_x\geq 0$
  and any examples $(x_i)_{i=1}^n$
  with $\max_i \|x_i\|_*\leq B_x$, then
  \[
    \Rad\del{\cbr{ ( w^\T x_1,\ldots,w^\T x_n) : \|w\| \leq B }} \leq \frac {C_6 B_x B}{\sqrt n}.
  \qedhere
  \]
\end{definition}

For many cases of interest, such as a Euclidean norm $\|\cdot\|_2$, we may simply take
$C_6 = 1$; however, even for seemingly simplistic cases like $\|\cdot\|_1$,
we need $C_6 := \sqrt{2 \ln(d)}$, which is not only larger than $1$, but moreover is
dimension-dependent.
The interested reader is directed to \citep{kakade2008complexity} for many cases where
\Cref{defn:rad_norm} can be worked out.
Overall there seems no simple one-size-fits-all way to control Rademacher complexity in this
general setting of MD, and therefore this constitutes an interesting separation between the
stochstic and batch methods presented herein, which will be revisited in \Cref{sec:open}.

With this definition out of the way, the bound for batch MD is as follows.

\begin{theorem}\label{fact:md:batch}
  Suppose $\ell$ is convex and $(C_1,C_2)$-quadratically-bounded.
  Suppose
  $((x_i,y_i))_{i\leq n}$ are drawn IID with $\max\{\|x_i\|_*,|y_i|\} \leq 1$ almost surely,
  and that $(w_i)_{i\leq t}$ are given by batch MD with $t \leq n$,
  where batch gradients are given \cref{eq:batch_md}.
  Let reference solution $\wref$ and initial point $w_0$ be given,
  and suppose $\wref$ satisfies $\cE(\wref) \leq D_\psi(\wref,w_0) / \sqrt{t}$.
  Then with probability at least $1-4\delta$, every $i\leq t$ satisfies
  \begin{align*}
    \frac {1}{i\eta} D_\psi(\wref, w_i)
    +
    \frac 1 i \sum_{j<i} \cR (w_j)
    &\leq
    \frac {B_w^2}{8i\eta}
    +
    \cR(\wref),
  \end{align*}
  where $B_w = \max\big\{ 1, \1[C_2>0] \|\wref\|, 4\sqrt{ D_\psi(\wref,w_0)}\big\}$
  and $\eta \leq \frac{1}{ 4096\max\{1,C_1, C_2\} \sqrt{t(C_6 +6\ln(1/\delta))} } $.
\end{theorem}

Other than replacing Markov chain concentration with a Rademacher complexity bound
powered by \Cref{defn:rad_norm}, the proof is essentially the same as the proof of \Cref{fact:md:general}.
Lastly, as a brief note on prior work,
while the main focus of this work is on methods which process one example at a time,
general unconstrained batch guarantees as above similarly do not seem to have appeared
in the literature; as with the stochastic analyses, the closest prior work has rates depending
on structural properties of the loss and training data \citep{nati_logistic,min_norm}.

To close this section, the batch analysis will also be extended to the case of
\emph{continuous time}, specifically the \emph{mirror flow (MF)}, which can be motivated
and defined as follows, but just as with MD is due to \citet{nemirovsky_yudin}.

To start the development of MD,
it is useful to first reformulate MD in terms of \emph{dual variables $(q_i)_{i\geq 0}$}.
To start, applying the first-order optimality conditions to the MD update,
a vector $w_{i+1}$ is a valid MD iterate iff
\begin{equation}
  \eta g_{i+1} + \nabla \psi(w_i) - \nabla\psi(w_{i+1}) = 0.
  \label{eq:md:optcond}
\end{equation}
In this duality-based construction, we will reverse the direction of our original MD
construction in \cref{eq:md:1}: rather than using optimization to define $w_{i+1}$ and then
saying it satisfies \cref{eq:md:optcond}, instead we will choose $w_{i+1}$ satisfying
\cref{eq:md:optcond} which in turn implies it is a valid MD update.
Specifically, consider the alternative 
sequence $((w_i,q_i))_{i\geq 0}$ where $w_0$ is given as before, but now additionally
define $q_0 := \nabla \psi(w_0)$, and thereafter
\begin{align*}
  g_{i+1} 
\in \partial f_{i+1}(w_i),
  \qquad q_{i+1}
:= q_i - \eta g_{i+1},
  \qquad w_{i+1} := \nabla \psi^*(w_{i+1}),
\end{align*}
where for technical reasons we will assume $\psi$ is \emph{Legendre}, which is
enough to guarantee $\nabla\psi$ and $\nabla\psi^*$ exist whenever encountered
\citep[Section 26.4]{lattimore_szepesvari_2020}.
By construction (and the equality case of Fenchel-Young),
these iterates satisfy the MD update optimality conditions in \cref{eq:md:optcond},
and thus correspond to a valid sequence of MD iterates.
Furthermore, by reorganizing the update for dual variable $q_{i+1}$,
we obtain
\[
  \frac { q_{i+1} - q_i}{\eta} = - g_{i+1},
\]
which after taking $\lim_{\eta\downarrow 0}$ suggests
\[
  \frac \dif {\dif s}q_s = \dot q_s \in - \partial f(w_s),
\]
where the time derivative notation $\dot q_s$ will be used throughout.
With this motivation in hand, we can defined MF as the solution to the following
differential equation:
letting initial choices $w_0$ and $q_0 := \nabla \psi(w_0)$
be given,
require
\begin{align}
w_s
= \nabla \psi^*(q_s),
  \qquad g_s \in \partial f(w_s),
  \qquad \dot q_s = - g_s,
  \label{eq:mf}
\end{align}
or more succinctly
\[
  \dot q_s \in - \partial f\del{\nabla \psi^*(q_s)}.
\]
This is not a differential equation but a \emph{differential inclusion};
we will tacitly assume existence and uniqueness of solutions (in particular,
a solution $(q_s)_{s\geq 0}$ which is well-behaved for almost all times $s$),
which is reasonable in this setting since $f$ is convex and $\psi$ is Legendre.

\begin{theorem}\label{fact:mf:batch}
  Suppose $\ell$ is convex and $(C_1,C_2)$-quadratically-bounded,
  and that $\psi$ is Legendre.
  Suppose
  $((x_i,y_i))_{i\leq n}$ are drawn IID with $\max\{\|x_i\|_*,|y_i|\}\leq 1$ almost surely,
  $\|\cdot\|$ is $C_6$-Rademacher,
  and that $(w_s)_{s\leq t}$ are given by batch MF in \cref{eq:mf},
  with the empirical risk $\hcR$ as the objective function (cf. \cref{eq:batch_md}).
  Let reference solution $\wref$ and initial point $w_0$ be given,
  and suppose $\wref$ satisfies $\cE(\wref) \leq D_\psi(\wref,w_0) / \sqrt{n}$.
  Then with probability at least $1-4\delta$, every $s\leq t$ satisfies
  \begin{align*}
    \frac 1 s
    D_\psi(\wref, w_{s})
    +
    \frac 1 s
    \int_0^s \cR(w_r)  \dif r
    &\leq
\frac {B_w^2 }{2s}
    +
    \cR(\wref)
    ,
  \end{align*}
  where $B_w = 4 \max\big\{ 1, \1[C_2 > 0] \|\wref\|, \sqrt{ D_\psi(\wref,w_0)}\big\}$
  and
  $t\leq \frac {\sqrt{n}}{16 \max\{1,C_1,C_2\}\del{C_6 + 6\sqrt{\ln(1/\delta)}} }$.
\end{theorem}

The form of the statement is essentially analogous to \Cref{fact:md:batch}; some
complexity of the step size $\eta$ is now pushed into the final time $t$.

A variety of comments on the proof are in order, however.  To invoke the coupling-based
proof scheme, the flow $(q_s)_{s\geq 0}$ would need to be tied to a projected flow
$(p_s)_{s \geq 0}$.  While this is entirely possible, the usual way of handling the constraint
would be to bake it into $\psi$, which introduces many technicalities when verifying
the existence and uniqueness of the flow owing to the resulting nonsmoothness of $\psi$.
This could be navigated, but instead, \Cref{fact:mf:batch} is proved via a shortcut argument
that \emph{does not} use coupling.  Similarly, it seems that the stochastic bounds in the
other parts of this work, rather than using the coupling argument, could have been proved
via some heretofore undiscovered powerful martingale concentration inequality, which
here is provided by a Rademacher generalization bound; this point will be revisited in
\Cref{sec:open}.

\section{Further related work and open problems}
\label{sec:open}

\paragraph{Implicit regularization.}
The extensive literature on implicit bias/regularization
was a major source of techniques and inspiration
for the present work.
These works typically show convergence to a specific (limiting) good solution
over the training set; this has been shown for coordinate descent
\citep{zhang_yu_boosting,mjt_margins}.
gradient descent \citep{nati_logistic,min_norm},
deep linear networks \citep{align,arora2019implicit},
ReLU networks \citep{kaifeng_jian_margin,chizat_bach_imp,dir_align},
mirror descent \citep{GLSS18}, and many others.
One point of contrast is that the focus in the preceding works is on structure
of the training set, not the structure of the population distribution as is used here.
Moreover, the relaxed criterion here (where there is no effort to prove $w_t \to \wref$
in any topology) allows treatment of previously difficult cases, such as the
spherical zero-margin data in \Cref{sec:homotopy}.

Is there a more refined comparison between the approaches?  Is there a stronger convergence
property over the distribution than the ones here?
\Cref{fig:intro} suggests stochastic GD concentrates along the path of population GD;
is there an easy way to prove this, perhaps via implicit regularization techniques
or the techniques in the present work?

That said, there is an increasing body of work which takes the view of mirror descent here,
namely of not dropping the term $D_\psi(\wref,w_t)$ and treating it as crucial
\citep{min_norm,shamir2021gradient,vaskevicius2020statistical}.
There is also work on the unbounded \emph{online} setting \citep{pmlr-v75-cutkosky18a},
but the guarantees there contain gradient norms and other terms which require further 
assumptions and analysis (as in the present work, in a statistical setting)
to be ensured small.

\paragraph{Deep networks.}
Deep learning is a natural target for the techniques presented here.
Unfortunately, the proofs heavily rely upon convexity.
Is there some adjustment that can be made to hold for general deep
networks, meaning those far outside the initial linearization regime?

\paragraph{Concentration-based proof technique.}
Is there a way to prove these results without needing a coupled sequence $(v_i)_{i\leq t}$?
There is already evidence of this in the proofs of \Cref{sec:md:batch}, specifically the
proof of \Cref{fact:mf:batch}, which preferred a less technical proof avoiding coupling.
Is there a powerful concentration inequality
which can directly establish concentration along the population GD path?

\paragraph{SGD vs GD.}
Many recent works aim to exhibit and study cases where SGD behaves \emph{differently}
from GD, with an eye towards giving further justification to the extensive use
of SGD in practice \citep{wu2020direction}; in this sense, the present work is a bit
unambitious. Is there some way to use the coupling-based proof technique here --- perhaps
by coupling with a very different path --- to establish other behaviors of SGD?
Relatedly, is there a powerful concentration inequality which can reduce the discrepancies
between the batch and stochastic cases herein, e.g., dropping the need for
\Cref{defn:rad_norm} in \Cref{sec:md:batch}?

\maketitle

\subsubsection*{Acknowledgments}
The author thanks
Daniel Hsu,
Ziwei Ji,
Liam O'Carroll,
Francesco Orabona, 
Maxim Raginsky,
Jeroen Rombouts, and Danny Son for valuable discussions,
as well as the COLT 2022 reviewers for many helpful comments, specifically regarding readability.
The author thanks the NSF for support under grant IIS-1750051.

\bibliographystyle{plainnat}
\bibliography{bib}

\appendix

\section{Technical preliminaries}

This first appendix proves basic properties of the loss functions considered,
then proves a variety of concentration inequalities,
and lastly provides the proofs for the examples in \Cref{sec:homotopy}.

\subsection{Losses}
\label{sec:losses}

First, the proof of \Cref{fact:qb:lip_or_smooth}, that Lipschitzness or smoothness suffice for
quadratic-boundedness.

\begin{proofof}{\Cref{fact:qb:lip_or_smooth}}
  For both assumptions, it is easiest to consider classification and regression losses separately.
  If $\ell$ is $\alpha$-Lipschitz,
  then if it is a classification loss $|\partial \ell(y,\yhat)| = |\partial \tilde \ell(\sgn(y)\yhat)| \leq \alpha$,
  whereas for regression $|\partial \ell(y,\yhat)| = |\partial \tilde \ell(y-\yhat)| \leq \alpha$.
  For the case of a smooth loss, with classification
  \[
    |\partial \ell(y,\yhat)|
    \leq |\partial \tilde \ell(\sgn(y)\yhat) - \partial \tilde \ell(0)| + |\partial \tilde \ell(0)|
    \leq \beta |\yhat| + |\partial \tilde \ell(0)|,
  \]
  and for regression
  \[
    |\partial \ell(y,\yhat)|
    \leq |\partial \tilde \ell(y - \yhat) - \partial \tilde \ell(0)| + |\partial \tilde \ell(0)|
    \leq \beta |y - \yhat| + |\partial \tilde \ell(0)|.
  \]
\end{proofof}

Next, the special properties of the logistic and squared losses.

\begin{proofof}{\Cref{fact:self-bounding}}
  This proof is split into the two losses.
  \begin{enumerate}
    \item
      For the squared loss, $\tell''(z) = 1$ and $\tell'(z)^2 = z^2 = 2 \tell(z)$,
      implying $1$-smoothness and $1$-self-boundedness.  For $(0,1)$-quadratic-boundedness,
      it suffices to apply \Cref{fact:qb:lip_or_smooth} and note that $\partial \tell(0) = 0$.

    \item
      For the logistic loss,
      a standard calculations reveal $(1/4)$-smoothness via $\tell''(z) \leq \tell''(0) = 1/4$
      and $1$-Lipschitz via $|\tell'(z)|\leq 1$.
      The $(1/2)$-self-bounding property was stated in \citep{mjt_margins}.
      For $(1,0)$-quadratically-bounding, it suffices to \Cref{fact:qb:lip_or_smooth}; it
      is nicer to use the Lipschitz bound since it shrinks some of the bounds throughout this work.
  \end{enumerate}
\end{proofof}

Continuing with the basic development, here is the proof of \Cref{fact:uref},
giving a single regularized iterate we can always plug in for $\wref$.

\begin{proof}[Proof of \Cref{fact:uref}]
  Fix $\lambda > 0$ and write $\uref := \uref(\lambda)$ for convenience.
  The bulk of the proof is to show that there always exists
  some $\wref$ with $\cE(\wref) \leq \lambda D_\psi(\wref,w_0)/2$;
  to see that this completes the proof, note that
  \[
    \cE(\uref)
    \leq
    \cE(\uref)
    + \frac \lambda 2  D_\psi(\uref, w_0)
    \leq
    \cE(\wref)
    + \frac \lambda 2  D_\psi(\wref, w_0)
    \leq
    \lambda  D_\psi(\wref, w_0),
  \]
  as desired.

  To show $\wref$ exists, consider the sublevel set
  \[
    S := \cbr{ v \in \R^d : \cE(v) \leq \frac \lambda 2},
  \]
  and consider the following two cases.
  \begin{enumerate}
    \item
      Suppose $D_\psi(u,w_0) \leq 1$ for every $u\in S$; by strong convexity of $\psi$,
      it follows that $S$ is a compact convex set, and moreover must contain a minimizer
      $\barw \in S$, meaning $\cE(\barw) = 0 \leq \lambda D_\psi(\barw,w_0) / 2$.
      As a result, choosing $\wref := \barw$ completes this case.

    \item
      Suppose there exists $u\in S$ with $D_\psi(u,w_0) > 1$,
      and choose $\wref := u$.  Then
      \[
        \cE(\wref) \leq \frac \lambda 2 = \frac \lambda 2 \cdot 1 < \frac \lambda 2 D_\psi(\wref, w_0).
      \]
  \end{enumerate}
\end{proof}

Lastly, a few key consequences of the definition of quadratic-boundedness, which is
used in all proofs (except for TD).

\begin{lemma}\label{fact:quad_bound}
  Suppose $\ell$ is $(C_1,C_2)$-quadratically-bounded
  and let $B_0 \geq 1$ and $B_x \geq 0$ be given.
  Given any $(x,y)$ with $\max\{\|x\|_*,|y|\} \leq B_x$,
  and any $u,v$,
  \begin{align*}
    \|\partial \ell_{x,y}(u)\|_* 
    &\leq B_x \sbr{ C_1 + C_2 B_x \del{ 1 + \|u\|}}
    ,
    \\
    \|\ell_{x,y}(u) - \ell_{x,y}(v)\|
    &\leq B_x \|u-v\| \sbr{ C_1 + C_2 B_x \del{ 1 + \|u\|}}.
  \end{align*}
  In particular, given any reference point $\wref$
  and vectors $S := \{ u\in\R^d : \|u-\wref\| \leq B_0\}$,
  then every $u,v\in S$ satisfy
  \begin{align*}
    \|\partial \ell_{x,y}(u)\|_* 
    &\leq B_x \sbr{ C_1 + C_2 B_x( 2 B_0 + \|\wref\|) }
    ,
    \\
    \|\ell_{x,y}(u) - \ell_{x,y}(v)\|
    &\leq
    B_x \|u-v\| \sbr{ C_1 + C_2 B_x ( 2 B_0 + \|\wref\| ) }.
  \end{align*}
  Lastly, let $\hcR$ and $\cR$ denote the empirical and population risks as in \Cref{sec:md:batch},
  and suppose $\max\{\|x\|_*, |y|\} \leq B_x$ almost surely.
  If $\|\cdot\|$ is $C_6$-Rademacher (cf. \Cref{defn:rad_norm}),
  then with probability at least $1-4\delta$, every $u\in S$ satisfies
  \[
    \envert{\hcR(\wref) - \hcR(u) - \cR(\wref) + \cR(u)}
    \leq
    \frac {B_x B_0\del{ C_1 + B_x C_2\del{2B_0 + \|\wref\| } }}{\sqrt n}
    \del{
      C_6 + 6\sqrt{\ln(1/\delta)}
    }.
  \]
\end{lemma}
\begin{proof}
  For the first inequality, by the chain rule,
  \[
    \|\partial_u \ell(y, x^\T u)\|_*
    =
    \|x \ell' (y, x^\T u)\|_*
    \leq
    B_x \envert{ C_1 + C_2 \del{ |y| + \big|x^\T u\big| } }
    \leq 
    B_x \sbr{ C_1 + C_2 B_x  \del{ 1 + \|u\| } }.
  \]
  For the second inequality, using a version of the fundamental theorem of calculus for subdifferentials
  \citep[Theorem D.2.3.4]{HULL},
  \begin{align*}
    \envert{ \ell_{x,y}(u) - \ell_{x,y}(v) }
    &=
    \Big| \int_0^1 \ip{\partial \ell_{x,y}(v + t(u-v))}{u-v}\dif t \Big|
    \\
    &\leq
    \int_0^1 \|\partial \ell_{x,y}(v + t(u-v))\|_* \|u-v\| \dif t
    \\
    &\leq
    B_x \|u-v\| \int_0^1 \sbr{ C_1 + C_2 B_x ( 1+ \|v + t(u-v)\| ) } \dif t
    \\
    &\leq
    B_x \|u-v\| \sbr{ C_1 + C_2 B_x ( 1 + \|v\| + \|u-v\| / 2 ) }.
  \end{align*}
  For the bounds with $\wref$ and $B_0$, invoking the previous two bounds
  with $u\in S$ and $v\in S$ for the second bound and $v := \wref$ for the first gives
  \begin{align*}
    \|\partial_u \ell(y, x^\T u)\|_*
    &
    \leq 
    B_x \sbr{ C_1 + C_2 B_x  \del{ 1 + \|u - \wref \| + \| \wref\|} }
    \\
    &
    \leq
    B_x \sbr{ C_1 + 2 C_2 B_x B_0 + C_2 B_x \| \wref\|}
    ,
    \\
    \envert{ \ell_{x,y}(u) - \ell_{x,y}(v) }
    &\leq
    B_x \|u-v\| \sbr{ C_1 + C_2 B_x ( B_0 + \|\wref\| + B_0 / 2 ) }
    \\
    &\leq
    B_x \|u-v\| \sbr{ C_1 + 2 C_2 B_x B_0 + C_2 B_x  \|\wref\| }
    .
  \end{align*}
  as desired.

  The Rademacher bounds are a consequence
  of standard tools from Rademacher complexity, detailed momentarily,
  and the preceding repackaged as follows:
  defining $C_3 := C_1 + B_x C_2\del{2B_0 + \|\wref\| } $
for convenience,
  and
  recalling the definitions of $\hcR$ and $\cR$,
  then for any $u,v\in S$ and any $x,y$ with $\max\{ \|x\|_*, |y|\} \leq B_x$,
  \begin{align*}
    |\ell_{x,y}'(u)|
    &\leq
    C_1 + C_2\del{|y| + |x^\T u| }
\leq 
C_3,
    \\
    |\ell_{x,y}(u) - \ell_{x,y}(\wref)|
    &\leq
    B_x B_0 C_3.
  \end{align*}
  Combining these inequalities
  with the fundamental theorem of Rademacher complexity
  \citep[Theorem 26.5]{shai_shai_book}
  the Lipschitz composition lemma \citep[Lemma 26.9]{shai_shai_book},
  and the $C_6$-Rademacher property,
  with probability at least $1-4\delta$,
  simultaneously for every $u\in S$,
  \begin{align*}
    \envert{\hcR(\wref) - \hcR(u) - \cR(\wref) + \cR(u)}
    &\leq
    \textup{Rad}\del{\cbr{(\ell_1(u),\ldots,\ell_n(u) : u \in S_\psi)}}
    + 6 B_x B_0 C_3 \sqrt{\frac {\ln(1/\delta)}{2n}}
    \\
    &\leq
    C_3 \textup{Rad}\del{\cbr{(x_1^\T u,\ldots,x_n^\T u) : u \in S_\psi)}}
    + 6 B_x B_0 C_3 \sqrt{\frac {\ln(1/\delta)}{2n}}
    \\
    &\leq
    \frac {B_x B_0 C_3}{\sqrt n}
    \del{
      C_6 + 6\sqrt{\ln(1/\delta)}
    }.
  \end{align*}
\end{proof}

\subsection{Concentration inequalities}
\label{sec:conc}

The first concentration inequality is a tiny bit of algebra on top of a convenient
reformulation (with elementary proof) of Freedman's inequality due to \citet{djh_mini_monster}.
This concentration inequality will be used to get $1/t$ rates in realizable settings.

\begin{lemma}\label[lemma]{fact:freedman:convenient}(See also \citet{djh_mini_monster}, Lemma 9.)
  Let nonnegative random variables $(X_1, \ldots, X_t)$ be given with
  $|X_i| \leq B$.
  Then for any $c\geq 4$, with probability at least $1-\delta$,
  \[
    \sum_{i=1}^t \sbr{ X_i - \Ex_{<i} X_i }
    \leq \frac 1 c \sum_{i=1}^t \Ex_{<i} \envert{ X_i } + cB \ln\del{\frac 1 \delta}.
  \]
\end{lemma}
\begin{proof}
  For each $i$, define $Y_i := X_i - \Ex_{<i} X_i$,
  whereby $\Ex_{<i} Y_i = 0$,
  and $|Y_i|\leq 4 (e-2) B \leq c(e-2)B$,
  and
  \[
    \Ex_{<i} Y_i^2
    =
    \Ex_{<i}\del{ X_i - \Ex_{<i}X_i}^2
    =
    \Ex_{<i}X_i^2 - \del{\Ex_{<i}X_i}^2
    \leq
    \Ex_{<i}X_i^2
    \leq
    B \Ex_{<i}\envert{ X_i }.
  \]
  As such, by a version of Freedman's inequality \citep[Lemma 9]{djh_mini_monster},
  with probability at least $1-\delta$,
  \[
    \sum_{i=1}^t Y_i
    \leq \frac {1}{cB} \sum_{i=1}^t \Ex_{<i} Y_i^2 + (e-2) cB \ln(1/\delta)
    \leq \frac {1}{c} \sum_{i=1}^t \Ex_{<i} \envert{ X_i } + cB \ln(1/\delta).
  \]
\end{proof}

Next comes a concentration inequality for Markov chains mentioned in the body.
The proof is based on a very nice one due to 
\citet[Proposition 1]{duchi2012ergodic},
though re-organized and fully decoupled from mirror descent;
e.g., this same bound will be used in the TD proofs.
That said, all the core ideas are from
\citep[Proposition 1]{duchi2012ergodic}.

\begin{lemma}\label[lemma]{fact:conc:markov}(See also \citet{duchi2012ergodic}, Proposition 1.)
Let $\eps\geq 0$ be given along with approximate stationarity witness $(\pi,\tau,\eps)$
  on a stochastic process $(x_i)_{i\leq t}$.
Let $f$ be given with $|f(x;w)|\leq B_f$ almost surely,
  and suppose
  $\envert{ f(x_{i+1};w_i) - f(x_{i+1};w_{i-\tau + 1})} \leq B_i$ almost surely for all $i$.
  With probability at least $1-\tau\delta$,
  \[
    \sum_{i<t}\sbr{ f(x_{i+1}; w_i)
      -
      \Ex_{x\sim\pi} f(x;w_i)
    }
    \leq
    2B_f\del{2\tau - 2 + t\eps + \sqrt{t\tau \ln(1/\delta)}}
    +\sum_{i=\tau-1}^{t-1} B_i.
  \]
\end{lemma}

A notable characteristic of this bound (also present in the version due to
\citet{duchi2012ergodic}) is that in the IID setting with $\tau=1$ and $\eps\approx0$,
the bound is exactly what one would expect from Azuma's inequality, with no excess.

\begin{proof}
  The key idea of the proof is to introduce gaps of length $\tau$ wherever
  $x_i$ and $w_i$ interact, making them approximately independent.
  To accomplish this,
  following a proof idea due to \citet{duchi2012ergodic}, the summation over time
  will be replaced with $\tau$ interleaved summations,
  where $w_i$ only interacts with $x_{i+\tau}$, meaning enough time has been inserted to
  ensure mixing.

  Before proceeding with the bulk of the proof, let's dispense with the nuisance case
  $t< 2\tau$.  If $t=0$, the bound is direct,
  and if $t >0$,
  by the definition of $B_f$, almost surely
  \begin{align*}
    \sum_{i<t}\sbr{ f(x_{i+1}; w_i)
      -
      \Ex_{x\sim\pi} f(x;w_i)
    }
    &\leq 2t B_f
    \leq 2B_f (2\tau - 1)
    \leq 2B_f (2\tau - 2 + \sqrt{t\tau\ln(1/\delta)}).
  \end{align*}
  which is upper bounded by the final desired quantity and completes the proof in the
  case $t<2\tau$.
  For the remainder of the proof, suppose $t\geq 2\tau$.

  Due essentially to boundary conditions, a bit of additional notation
  and care are needed, especially to ensure that the bound loses nothing when
  $\tau = 1$ (e.g., the IID case).
  Let $\rho\geq \tau - 1$ denote the smallest time so that $\tau | (t - \rho)$
  and let $n := (t - \rho)/\tau$; this $\rho$ is the number of iterates that will be thrown
  out so that the overall sum can be split into $n$ interleaved sums.
  Since $t \geq 2\tau$, then $\rho$ always exists and satisfies $\tau-1 \leq \rho \leq 2(\tau-1)$
  (if $t-(\tau -1)$ is not a multiple of $\tau$, then there must be a multiple
  of $\tau$ within $\{t - 2(\tau -1) , \ldots, t-\tau \}$),
  and $t / (2\tau)\leq  n \leq t/\tau$.

  To start, defining $B_\tau := \sum_{i=\tau-1}^{t-1} B_i$ for convenience,
  \begin{align}
    \sum_{i<t} f(x_{i+1}; w_i)
    &=
    \sum_{i<\tau-1} f(x_{i+1}; w_i)
    \notag\\
    &\quad
    +
    \sum_{i=\tau-1}^{t-\rho + \tau - 2} \del{
      f(x_{i+1};w_i)
      - f(x_{i+1}; w_{i - \tau + 1})
      + f(x_{i+1}; w_{i - \tau + 1})
    }
    \notag\\
    &\quad
    +
    \sum_{i = t - \rho + \tau - 1}^{t-1}
    f(x_{i+1};w_i)
    \notag\\
    &\leq
    \rho B_f
    + B_\tau
    +
    \sum_{i=0}^{t-\rho - 1}
    f(x_{i+\tau};w_i)
    \notag\\
    &\leq
    \rho B_f
    + B_\tau
    +
    \sum_{i=0}^{t-\rho - 1}
    \Ex_i f(x_{i+\tau};w_i)
    +
    \sum_{i=0}^{t-\rho - 1}
    \del{ 
    f(x_{i+\tau};w_i)
    - \Ex_i f(x_{i+\tau};w_i)}
    \notag\\
    &\leq
    2 \rho B_f
    + B_\tau
    +
    \sum_{i<\tau} \Ex_{x\sim \pi} f(x;w_i)
    \notag\\
    &\quad
    +
    \sum_{i=0}^{t-\rho - 1}
    \del{ \Ex_i f(x_{i+\tau};w_i) - \Ex_{x\sim\pi} f(x;w_i) }
    \notag\\
    &\quad
    +
    \sum_{i=0}^{t-\rho - 1}
    \del{ 
    f(x_{i+\tau};w_i)
    - \Ex_i f(x_{i+\tau};w_i)}.
    \label{eq:markov:fullsum}
  \end{align}
  The rest of the proof will handle these last two summations, the first via the definition
  of $\tau$ and the mixing properties of the stochastic process,
  and the second via concentration inequalities and the aforementioned interleaved sum technique.

  For the first summation, fix any $i$,
  and let $\xi_i$ be a coupling between the distribution of $x_{i+\tau}$ 
  conditioned on $\cF_i$, and the stationary distribution $\pi$.
  By the coupling characterization of total variation
  \citep[Equation 6.11]{villani_2},
  and since $\tv(P_{i}^{i+\tau}, \pi) \leq \eps$,
\begin{align}
    \Ex_i f(x_{i+\tau};w_i)
    -
    \Ex_{x\sim \pi} f(x;w_i)
    &= 
    \Ex_{(y,x)\sim \xi_i}
    \del{ f(y;w_i) - f(x;w_i) }
    \notag\\
    &\leq
    2B_f\Pr_{(y,x)\sim \xi_i} [ x \neq y ]
    \notag\\
    &\leq
    2B_f\eps.
    \label{eq:markov:tv}
  \end{align}
  Summing these errors adds a term $2\eps(t-\rho) B_f$ to \cref{eq:markov:fullsum}.

  The final summation in \cref{eq:markov:fullsum} will utilize the aforementioned technique
  of splitting the sum into $\tau$ interleaved sums.  Concretely,
  \[
    \sum_{i=0}^{t-\rho - 1}
    \del{f(x_{i+\tau};w_i) - \Ex_i f(x_{i+\tau};w_i)}
    =
    \sum_{j=0}^{\tau - 1}
    \sbr{
      \sum_{k=0}^{n-1}
      \del{f(x_{k\tau + j + \tau};w_{k\tau + j}) - \Ex_i f(x_{k\tau + j +\tau};w_{k\tau + j})}
    },
  \]
  and all that remains is to apply a concentration inequality to the $\tau$ inner summations,
  and collect terms.
  By $\tau$ applications of Azuma's inequality, with probability at least $1-\tau \delta$,
  simultaneously for all $j \in \{0,\ldots,\tau - 1\}$,
  \[
    \sum_{k=0}^{n-1}
    \del{f(x_{k\tau + j + \tau};w_{k\tau + j}) - \Ex_i f(x_{k\tau + j +\tau};w_{k\tau + j})}
    \leq
    B_f \sqrt{2n\ln(1/\delta)}.
  \]
  Combining all this with \cref{eq:markov:fullsum}, the total variation bound
  in \cref{eq:markov:tv}, and simplifying to replace $\rho$ and $n$ with their
  respective bounds gives
  \begin{align*}
    \sum_{i<t}\sbr{ f(x_{i+1}; w_i)
      -
      \Ex_{x\sim\pi} f(x;w_i)
    }
    &\leq
    2 (\rho + \eps (t-\rho))B_f
    + B_\tau
    + B_f\tau \sqrt{2n\ln(1/\delta)}
    \\
    &\leq
    2(2\tau - 2 + t\eps) B_f
    + B_\tau
    + B_f\sqrt{2t\tau \ln(1/\delta)},
  \end{align*}
  which completes the proof.
\end{proof}

The last collection of concentration inequalities needed are for heavy-tailed losses.
There appear to be many such inequalities,
for example here is one due to \citet{bhk_data_science}.

\begin{lemma}(See also \citep{bhk_data_science}.)
Let IID random variables $(Z_i,\ldots,Z_t)$ be given
  with variance at most $\sigma^2$,
  and suppose there exists $m \leq t\sigma^2 /2$
  with $\envert{\Ex \del{\|y_ix_i\|^2 - \Ex \|y_ix_i\|^2}^r}\leq \sigma^2 r!$
  for $r\in\{3,\ldots,m\}$.
  If $\delta \geq 1 / (n\sigma^2)^{s/2}$,
  then with probability at least $1-\delta$,
  \[
    \envert{ \sum_i \del{ Z_i - \Ex Z_i^2 } } \leq \frac {\sigma \sqrt{2sn}}{\delta^{1/s}}.
  \]
\end{lemma}
\begin{proof}
  This version performs minor algebra to repackage the original.
\end{proof}

The conditions on the preceding are a little complicated, so here is a potentially worse
bound with many fewer conditions to check.
It is presented in \citep{dt_kmbounds},
but follows a proof scheme due to \citet[Equation 7]{tao_conc},
making adjustments to drop boundedness assumptions.

\begin{lemma}\label{fact:conc:poly}(Appears in \citep{dt_kmbounds}
following a proof scheme due to \citet[Equation 7]{tao_conc}.)
Let IID random variables $(Z_i,\ldots,Z_t)$ be given.
  Suppose $p$ is even, and define $M := \max\{p/e, \sup_{2\leq r\leq p} \Ex \envert{ Z_i - \Ex Z_i }^r\}$.
  Then with probability at least $1-\delta$,
\[
    \envert{\sum_{i=1}^t \del{ Z_i - \Ex Z_i} }
    \leq
2M\sqrt t
    \del{\frac 2 \delta}^{1/p}.
  \]
\end{lemma}
\begin{proof}
  This is the same as \citep[Lemma A.3]{dt_kmbounds}, except the requirement $M \geq p/e$
  simplifies the bound, in particular the requirement $t\geq p/(Me)$ is now simply
  $t\geq 1$.
\end{proof}

\subsection{Analysis of examples in \Cref{sec:homotopy}}

To conclude this section of technical preliminaries, here are full proofs corresponding to
the illustrative examples presented in \Cref{sec:homotopy}.  First is the simple
margin-like bound.

\begin{proofof}{\Cref{fact:md:margin}}
First note that if $u^\T xy \geq \gamma$, then $\wref^\T xy \geq \ln t$;
  otherwise, $\wref^\T xy \geq -\|\wref\| \geq -\ln(t)/\gamma$ almost surely.
  Using the provided choice of $\wref$ and the elementary upper bounds
  $\ell(y,\yhat) = \ln(1+\exp(-y\yhat)) \leq \exp(-y\yhat)$ when $y\yhat \geq \gamma$
  and $\ell(y,\yhat) \leq 1 - y\yhat$
  when $y\yhat \leq 0$
  gives
  \begin{align*}
    \Ex_{x,y} \ell( y, x^\T \wref)
    &
    \leq
    \tell\del{\ln(t)}\Pr[u^\T x y \geq \gamma]
    +
    \tell\del{-\ln(t)/\gamma} \Pr[u^\T xy < \gamma]
    \\
    &
    \leq
    \exp(-\ln(t))
    + \del{1 + \frac {\ln t}{\gamma}}\frac {1}{t}
    \\
    &
    \leq
    \frac {1}{t} \del{2 + \frac {\ln(t)}{\gamma}}.
  \end{align*}
\end{proofof}

Next is the characterization of the optimal path corresponding to the sphere
data depicted in \Cref{fig:circle:log}.

\begin{proofof}{\Cref{fact:margin:zero}}
  The proof first establishes that $u_r := re_1$ (for $r>0$)
  is the unique risk minimizer with norm $r$,
  which will be established via symmetry argument.
  Consider any other solution $v$ with $\|v\| = r$.
  To avoid dealing with labels,
  let $\mu_1$ denote the density on the points $\{z \in \cS_{d-1} : z_1 \geq 0\}$
  obtain by sampling $(x,y)$ and then multiplying to obtain $z := xy$.
  For any $z$ in the support of $\mu_1$,
  consider the behavior of $v$ on $z$ and $z' := 2 z_1 e_1 - z$,
  which reflects $z$ around $z_1e_1$.
  With probability $1$, $z \neq e_1$,
  and thus $z \neq z'$, and by strict convexity
  \begin{align*}
    \frac 1 2 \ln(1+\exp(-v^\T z)) + \frac 1 2  \ln(1+\exp(-v^\T z'))
    &> \ln(1+\exp(-v^\T (z + z')/2))
    \\
    &= \ln(1+\exp(-v^\T e_1 z_1))
    \\
    &\geq \ln(1+\exp(-u_r^\T z)).
  \end{align*}
  Since $z$ and $2z_1 e_1 - z$ have equal probability density,
  letting $\mu_2$ denote the probability density obtained from $\mu_1$ by conditioning
  on $z_2>0$,
  \begin{align*}
    \cR(v)
    &= \int \ln(1+\exp(-v^\T z))\dif\mu_1(z)
    \\
    &=\int \ln(1+\exp(-v^\T z))\dif\mu_1(z)
    \\
    &=\frac 1 2 \int \del{ \ln(1+\exp(-v^\T z)) + \ln(1+\exp(-v^\T (2z_1e_1 - z)))}\dif\mu_2(z)
    \\
    &> \int \ln(1+\exp(-u_r^\T z))\dif\mu_2(z)
    \\
    &= \frac 1 2 \int\del{\ln(1+\exp(-u_r^\T (2z_1e_1 - z))) + \ln(1+\exp(-u_r^\T z))}\dif\mu_2(z)
    \\
    &= \int\ln(1+\exp(-u_r^\T z)) \dif\mu_1(z)
    \\
    &= \cR(u_r),
  \end{align*}
  establishing that $u_r$ is the unique optimal solution of norm $r$.

  Next, still letting $\rho$ denote the uniform measure on the hemisphere with combined
  variable $z$, to compute $\cR(u_r)$ exactly, recall the \emph{dilogarithm function}
  $\Li_2(z)$ \citep{enwiki:1057121753},
  which satisfies
  \[
    \Li_2(z) := \sum_{k \geq 1}\frac{z^k}{k^2},
    \qquad
    \frac \dif {\dif z} \Li_2(z) = -\ln(1-z).
  \]
  Then
  \begin{align*}
    \cR(u_r) = \int \ln(1+\exp(-u_r^\T z))\dif\rho(z)
    &= \int_0^1 \ln(1+\exp(-r s)) \dif s
    \\
    &=  - \frac 1 r \left.\Li(-\exp(-rs))\right|_0^1
    \\
    &= \frac 1 r \del{-\Li(-\exp(-r)) + \frac {\pi^2}{12}}.
  \end{align*}
  To control this further, $\Li(-\exp(-r))$ can be written
  \begin{align*}
    \Li(-\exp(-r))
    &=
    \sum_{k\geq 1} \frac {(-\exp(-r))^k}{k}
    \\
    &=
    \int_r^\infty \sum_{k\geq 1} \exp(-kr))
    \\
    &=
    \sum_{k\geq 1} \int_r^\infty \exp(-kr))
    \\
    &=
    \sum_{k\geq 1} \frac {\exp(-kr)}{k},
  \end{align*}
  which is at least $\exp(-r)$, implying
  \[
    \cR(u_r) \leq \frac 1 r \del{-\exp(-r) + \frac{\pi^2}{12}}.
  \]
  For the lower bound, if $r\geq 1$, then
  \begin{align*}
\sum_{k\geq 1} \frac {\exp(-kr)}{k}
&\leq \exp(-r) + \frac 1 2 \sum_{k\geq 2}\exp(-kr)
    \\
    &= \exp(-r) + \frac 1 2 \frac {\exp(-2r)}{1 - \exp(-r)}
    \\
    &= \exp(-r)\del{1 + \frac {1}{2(\exp(r) - 1)}}
    \\
    &\leq
    2\exp(-r),
  \end{align*}
  which completes the proof.
\end{proofof}

\section{Analysis of mirror descent (MD)}

This section collects all proofs related to mirror descent.

\subsection{Deterministic mirror descent analysis}

While the core deterministic proof for mirror descent is completely standard,
the standard presentation always omits the term $D_\psi(\wref,w_t)$, which as mentioned
is the basis for this entire work.  As such,
the standard guarantee is re-proved with this term included; the proof itself has otherwise
nothing new over other versions, and most closely follows one due to
\citet{duchi2012ergodic}.

\begin{lemma}\label[lemma]{fact:md}
  Let mirror map $\psi$,
  closed convex set $S$, initial iterate $w_0 \in S\cap \dom(\psi)$
  and reference solution $\wref\in S\cap\dom(\psi)$ be given.
  Let iterates $(w_i)_{i\geq 1}$ be obtained via MD on a
  sequence of objectives $(f_i)_{i\geq 1}$ with subgradients $g_{i+1} \in \partial f_{i+1}(w_i)$.
Then for all $i\leq t$,
  \[
    D_\psi(\wref, w_{t})
    \leq
    D_\psi(\wref,w_0)
    + \eta \sum_{i<t} \ip{\partial f_{i+1}(w_i)}{\wref - w_i}
+ \sum_{i<t} \frac{\eta^2}{2}\|g_{i+1}\|_*^2,
  \]
  and it holds for every $i<t$ that
  $\|w_{i+1} - w_i\| \leq \eta \|g_{i+1}\|_*$.
  If additionally each $f_i$ is convex, then
  \[
    D_\psi(\wref, w_{t})
    \leq
    D_\psi(\wref,w_0)
    + \eta \sum_{i<t} \sbr{f_{i+1}(\wref) - f_{i+1}(w_i)}
+ \sum_{i<t} \frac{\eta^2}{2}\|g_{i+1}\|_*^2.
  \]
\end{lemma}

\begin{proof}
  Fix any iteration $i$.
  By the first-order conditions on the choice of $w_{i+1}$,
  for any $v\in S$,
  \[
    \ip{\eta g_{i+1} + \nabla \psi(w_{i+1}) - \nabla \psi(w_i)}{v-w_{i+1} } \geq 0.
  \]
  Instantiating this first-order condition with $v=w_i$ rearranges to give
  \[
    \eta \|g_{i+1}\|_* \|w_i - w_{i+1}\|
    \geq \ip{\eta g_{i+1}}{w_i - w_{i+1}}
    \geq \ip{\nabla \psi(w_{i+1}) - \nabla \psi(w_i)}{w_{i+1} - w_i}
    \geq \|w_{i+1} - w_i\|^2,
  \]
  which implies $\|w_{i+1} - w_i\| \leq \eta \|g_{i+1}\|_*$ from the statement.  On the other hand, instantiating the first-order condition
  with $v=\wref$ gives
  \begin{align*}
    \ip{\nabla \psi(w_{i}) - \nabla \psi(w_{i+1})}{\wref-w_{i+1} }
    &\leq
    \ip{\eta g_{i+1}}{\wref-w_{i+1} }
    \\
    &\leq
    \eta \sbr{  \ip{g_{i+1}}{\wref-w_{i}} + \ip{g_{i+1}}{w_i-w_{i+1} } },
\end{align*}
  which combines with the definition of $D_\psi$ to give
  \begin{align*}
    D_\psi(\wref, w_{i+1}) - D_\psi(\wref, w_i)
    &=
    \psi(w_i) - \psi(w_{i+1}) - \ip{\nabla \psi(w_{i+1})}{\wref - w_{i+1}}
    + \ip{\nabla \psi(w_{i})}{\wref - w_{i}}
    \\
    &=
    -D_\psi(w_{i+1}, w_i) + \ip{\nabla \psi(w_{i}) - \nabla \psi(w_{i+1})}{\wref - w_{i+1}}
    \\
    &\leq
    \eta \ip{g_{i+1}}{\wref-w_{i}}
    +
    \eta \sbr{
    \ip{g_{i+1}}{w_i-w_{i+1} }
    -\frac 1 \eta D_\psi(w_{i+1}, w_i) }.
  \end{align*}
  To simplify this further,
  by the Fenchel-Young inequality and
  strong convexity of $D_\psi$,
  \begin{align*}
\ip{g_{i+1}}{w_i-w_{i+1} }
    - \frac 1 \eta D_\psi(w_{i+1}, w_i)
    &\leq
\frac{\eta}{2}\|g_{i+1}\|_*^2
    + \frac {1}{2\eta}\|w_i-w_{i+1}\|^2
    - \frac 1 \eta D_\psi(w_{i+1}, w_i)
\leq
\frac{\eta}{2}\|g_{i+1}\|_*^2.
  \end{align*}
  Since this inequality holds for any $i<t$,
  then applying $\sum_{i<t}$ to both sides,
  telescoping and rearranging gives
\[
    D_\psi(\wref, w_{t})
    \leq
    D_\psi(\wref,w_0)
    + \eta \sum_{i<t} \ip{g_{i+1}}{\wref-w_{i}}
    +
    \eta\sum_{i<t}
\frac{\eta}{2}\|g_{i+1}\|_*^2.
  \]
  Lastly, invoking convexity in the form
  $\ip{g_{i+1}}{\wref-w_{i}} \leq f_{i+1}(\wref) - f_{i+1}(w_i)$
  gives the convexity version of the claim.
\end{proof}

\subsection{The realizable case: \Cref{fact:md:realizable}}

This proof makes use of the variant of Freedman's inequality restated in \Cref{fact:freedman:convenient}.

\begin{proofof}{\Cref{fact:md:realizable}}
  As in the other proofs, let $(v_i)_{i\leq t}$ be coupled iterates projected onto
  the ball $S := \{ v \in \R^d : \|v-\wref\|\leq B_w \}$ with
  \[
    B_w := \max\cbr{ 1, 4\sqrt{ D_\psi(\wref,w_0)},\sqrt{\frac{64C_4}{\rho} \ln\frac 1 \delta}}
  \]
  (where $B = B_w \sqrt{1 + C_1 + C_2(1+   \|\wref\|) + C_4}$ is used in the statement),
  where $v_0 = w_0$ and
  common random data $((x_i,y_i))_{i=1}^t$ are used.  The proof will apply concentration
  inequalities on the common sample space, and then show that $(w_i)_{i\leq t} = (v_i)_{i \leq t}$
  and that both share good risk and norm guarantees.

  Unlike the other proofs, this realizable setting will apply two separate concentration
  inequalities in order to allow cleaner step sizes.  Concretely, the first concentration
  inequality will be on $\sum_{j<i}\ell_{j+1}(\wref)$ alone.
  Since $|\ell_{j+1}(\wref)| \leq C_4$ almost surely by assumption,
  then by $t$ applications of \Cref{fact:freedman:convenient} with constant $c=4$
  gives, with probability at least $1-t\delta$, simultaneously for all $i\leq t$,
  \begin{align*}
    \eta \sum_{j<i} \ell_{j+1}(\wref)
    &\leq
    \frac {5\eta i}{4} \cR(\wref) + 4 \eta C_4 \ln(1/\delta)
    =
    \frac {5\eta i}{4} \cR(\wref) +  \frac {B_w^2}{16}
    .
  \end{align*}
  The second concentration inequality will as usual
  involve both $\ell_{j+1}(\wref)$ and $\ell_{j+1}(v_j)$.
  To start, again using the constant $C_4$ but also \Cref{fact:quad_bound},
  and defining $C_5 := (C_4 + (C_1 + 2C_2 B_w + C_2 \|\wref\|)B_w)/2$ for convenience, for any $j<t$ and any
  $v \in S$,
  \begin{align*}
    \envert{2 \ell_{j+1}(\wref) - \ell_{j+1}(v)}
    \leq
    \envert{\ell_{j+1}(\wref)}
    +
    \envert{\ell_{j+1}(\wref) - \ell_{j+1}(v)}
    \leq
2C_5.
  \end{align*}
  Combining this bound with with another $t$ applications of \Cref{fact:freedman:convenient} with constant $c=4$,
  with probability at least $1-t\delta$, simultaneously for all $i\leq t$,
  \begin{align*}
    \sum_{j<i}\sbr{ \ell_{j+1}(\wref) - (1/2)\ell_{j+1}(v_j)}
    &\leq
    \sum_{j<i}\sbr{ (5/4)\cR(\wref) - (3/8)\cR(v_j)}
    + 4 C_5 \ln(1/\delta).
  \end{align*}
  For the remainder of the proof, discard the combined failure event from the preceding
  bounds, which together removes $2t\delta$ probability mass.

  The first concentration alone will now be used to prove $w_i = v_i \in S$ by induction;
  the base case $w_0=v_0 \in S$ is direct, thus consider
  some $i > 0$.
  By the concentration inequality on $\sum_{j<i}\ell_{j+1}(\wref)$,
  the fact that $\ell_{j+1}(w_j) \geq 0$,
  and using additionally $\cR(\wref)\leq \rho D_\psi(\wref,w_0) / t$,
  and lastly the definition of $\rho$-self-bounding,
  the deterministic mirror descent guarantee from \Cref{fact:md}
  becomes
  \begin{align}
    D_\psi(\wref, w_{i})
    &\leq
    D_\psi(\wref,w_0)
    + \eta \sum_{j<i} \sbr{ \ell_{j+1}(\wref) - \ell_{j+1}(w_j) }
    + \sum_{j<i} \frac{\eta^2 }{2}\ell_{j+1}'(w_j)^2\|x_{j+1}y_{j+1}\|_*^2
    \notag\\
    &\leq
    \frac {B_w^2}{4}
    + \eta \sum_{j<i} \sbr{ \ell_{j+1}(\wref) - (1-\eta\rho) \ell_{j+1}(w_j) }
    \label{eq:realizable:mid:1}\\
    &\leq
    \frac {B_w^2}{4}
    + \eta \sum_{j<i}\ell_{j+1}(\wref)
    \notag\\
    &\leq
    \frac {B_w^2}{4}
    + \frac {5i\eta}{4} \cR(\wref) + \frac {B_w^2}{16}
    \notag\\
    &\leq
    \frac {B_w^2}{4}
    + \frac {5D_\psi(\wref,w_0)}{8} + \frac {B_w^2}{16}
    \leq
    \frac {3B_w^2}{8},
    \notag
  \end{align}
  which establishes the desired norm control since $D_\psi(\wref,w_i) \geq \|w_i-\wref\|^2/2$,
  but also since $v_{i-1} = w_{i-1}$, then the construction of $v_i$ will not invoke the constraint
  and $v_{i} = w_i$.

  For the risk control, for any $i\leq t$, by the second concentration inequality above,
  using $(v_j)_{j<i} = (w_j)_{j<i}$ and continuing with the deterministic mirror descent
  bound from \cref{eq:realizable:mid:1},
  \begin{align*}
    D_\psi(\wref, w_{i})
    &\leq
    \frac {B_w^2}{4}
    + \eta \sum_{j<i} \sbr{ \ell_{j+1}(\wref) - (1/2) \ell_{j+1}(w_j) }
    \\
    &=
    \frac {B_w^2}{4}
    + \eta \sum_{j<i} \sbr{ \ell_{j+1}(\wref) - (1/2) \ell_{j+1}(v_j) }
    \\
    &\leq
    \frac {B_w^2}{4}
    + \eta \sum_{j<i} \sbr{ (5/4)\cR(\wref) - (3/8) \cR(v_j) }
    +  4\eta C_5\ln(1/\delta),
  \end{align*}
  which after expanding the choice of $C_5$ gives the desired bound.
\end{proofof}

\subsection{The non-realizable, Markov case: \Cref{fact:md:general}}

This proof makes use of the Markov chain concentration inequality in \Cref{fact:conc:markov}.

\begin{proofof}{\Cref{fact:md:general}}
  Let $(v_i)_{i\leq t}$ denote the coupled projected mirror descent iterates
  using projection ball $S := \{ v\in\R^d : \|v - \wref\| \leq B_w \}$,
  which are coupled in the strong sense that that $v_0 = w_0$,
  and thereafter $(v_i)_{i\leq t}$ and $(w_i)_{i \leq t}$
  use the exact same data sequence $((x_i,y_i))_{i=1}^t$.  As in the general proof
  scheme, the first step is to apply a concentration inequality on this shared 
  data $((x_i,y_i))_{i=1}^t$, and then show that $(v_i)_{i\leq t} = (w_i)_{i \leq t}$ by induction.

  Before starting on the general proof scheme, to simplify the interaction
  with the loss conditions, define $C_3 := C_1 + 2C_2 B_w + B_w$ for
  convenience, and note by \Cref{fact:quad_bound} and
  since $\max\{\|x\|_*,|y|\}\leq 1$, for any $u,v\in S$,
  \begin{align}
    \sup_{j<t}\|\partial \ell_{j+1}(u)\|_* 
    \leq
    C_3
    ,
    \qquad
    \sup_{j<t}
    \|\ell_{j+1}(u) - \ell_{j+1}(v)\|
    \leq
    C_3 \|u-v\|.
    \label{eq:qb:1}
  \end{align}
  Lastly, $C_3 \leq 4B_w \max \{1,C_1,C_2\}$.

  The concentration inequality will be the general stochastic process bound
  in \Cref{fact:conc:markov}, applied to
  $\eta \sum_{j<i} \sbr{\ell_{j+1}(\wref) - \ell_{j+1}(v_j)}$ for all $i\leq t$,
  which requires almost sure bounds on two quantities.
  The first is a uniform control on individual differences within this summation,
  which thanks to \cref{eq:qb:1} is simply
  \[
    \sup_{j<i} \envert{\ell_{j+1}(\wref) - \ell_{j+1}(v_j)} \leq C_3 B_w
    \qquad \textup{a.s.}.
  \]
  The second almost sure bound is on a similar difference but on iterates which are $\tau$ apart,
  which follows by combining \cref{eq:qb:1}
  with the
  per-iteration guarantee from \Cref{fact:md}:
  \begin{align*}
    \envert{\ell_{i+\tau+1}(v_{i+\tau}) - \ell_{i+\tau+1}(v_{i})}
    \leq
    C_3 \|v_{i+\tau} - v_i\|
    \leq
    C_3 \sum_{j=i}^{i+\tau - 1} \eta \|\nabla \ell_{j+1}(v_j)\|_*
    \leq
    \eta \tau C_3^2.
  \end{align*}
  As such, union bounding $t$ applications of \Cref{fact:conc:markov}
  with probability at least $1-t\tau\delta$,
  simultaneously for all $i \leq t$,
\begin{align*}
\eta \sum_{j<i} \sbr{\ell_{j+1}(\wref) - \ell_{j+1}(v_j)
    - \cR_{}(\wref) + \cR_{}(v_j)}
&\leq
    2 \eta C_3 B_w
    \sbr{2\tau + i\eps + \sqrt{i\tau\ln(1/\delta)}}
    + i \eta^2 \tau C_3^2.
  \end{align*}
  If $i\geq 2\tau$, then this bound simplifies to
  \begin{align*}
    \eta \sum_{j<i} \sbr{\ell_{j+1}(\wref) - \ell_{j+1}(v_j)
    - \cR_{}(\wref) + \cR_{}(v_j)}
    &\leq
    2 \eta C_3 B_w
    \sbr{\sqrt{i\tau} + \sqrt{i} + \sqrt{i\tau\ln(1/\delta)}}
    + i \tau \eta^2 C_3^2
    \\
    &\leq
    \frac {6B_w^2 + B_w^2}{1024}
\\
    &
    \leq
    \frac {B_w^2}{128}.
  \end{align*}
  On the other hand, if $i < 2\tau$,
  then forgoing
  \Cref{fact:conc:markov} entirely and using the almost
  sure bounds on the left hand side directly,
  \begin{align*}
\eta \sum_{j<i} \sbr{\ell_{j+1}(\wref) - \ell_{j+1}(v_j)
    - \cR_{\pi}(\wref) + \cR_{\pi}(v_j)}
\leq
    2 \eta i C_3 B_w
    \leq
    2 \eta \sqrt{ 2 i \tau } C_3 B_w
    \leq
    \frac {B_w^2}{128}.
  \end{align*}
  The remainder of the proof discards the common $t\tau\delta$
  failure probability on the underlying
  sample space $((x_i, y_i))_{i_1}^t$ which is shared by
  $(v_i)_{i\leq t}$ and $(w_i)_{i\leq t}$.

  The proof now proceeds by induction, establishing $v_i = w_i$ for $i\leq t$
  and the corresponding risk bound.
  The base case $w_0 = v_0$ is by definition of the coupling,
  thus consider the construction of some $w_i$ with $i>0$;
  by the deterministic mirror descent guarantee in \Cref{fact:md},
  the inductive hypothesis $(v_i)_{j< i} = (w_i)_{j< i}$, and the above
  concentration inequality on $(v_j)_{j< i}$,
  \begin{align*}
    D_\psi(\wref, w_i)
    &\leq
    D_\psi(\wref,w_0)
    + \eta \sum_{j<i} \sbr{ \ell_{j+1}(\wref) - \ell_{j+1}(w_j) }
    + \sum_{j<i}\frac {\eta^2}{2} \|\nabla \ell_{j+1}(w_j)\|_*^2
    \\
    &=
    D_\psi(\wref,w_0)
    + \eta \sum_{j<i} \sbr{ \ell_{j+1}(\wref) - \ell_{j+1}(v_j) }
    + \sum_{j<i}\frac {\eta^2}{2} \|\nabla \ell_{j+1}(v_j)\|_*^2
    \\
    &\leq
    \frac {B_w^2}{16}
    + \eta \sum_{j<i} \sbr{ \cR(\wref) - \cR (v_j) }
    + \frac {B_w^2}{128}
    + \frac {i\eta^2 C_3^2}{2}
    \\
    &\leq
    \eta \sum_{j<i} \sbr{ \cR(\wref) - \cR (v_j) }
    + \frac {B_w^2}{8},
  \end{align*}
  which establishes the risk guarantee on $(w_j)_{j<i}$
  after substituting $(v_j)_{j<i} = (w_j)_{j<i}$ back in.
  To verify $w_i = v_i \in S$,
  it suffices to establish $\|w_i-\wref\|< B_w$, which means $v_i$ will not encounter
  its projection and $w_i = v_i$; to this end,
  combining the preceding with the fact
$\cE(\wref) \leq \frac {D_\psi(\wref,w_0)} {\sqrt{t}}
  \leq \frac {B_w^2} {16\sqrt t}$,
  then
  \begin{align*}
    \frac 1 2 \|w_i - \wref\|^2
    \leq
    D_\psi(\wref, w_i)
    \leq
    \frac {B_w^2}{8}
    + \eta \sum_{j<i} \sbr{ \cR(\wref) - \cR (w_j) }
    \leq
    \frac {B_w^2}{8}
    + \eta \sum_{j<i} \frac {B_w^2} {16\sqrt{t}}
    \leq
    \frac {B_w^2}{4}
  \end{align*}
  as desired.
\end{proofof}

\subsection{Heavy-tailed data: \Cref{fact:md:heavy}}

This proof makes use of the heavy-tail concentration inequalities at the end of 
\Cref{sec:conc}.

\begin{proofof}{\Cref{fact:md:heavy}}
  This proof will follow the usual scheme --- applying concentration to projected iterates
  $(v_i)_{i\leq t}$ with $v_0 = w_0$ which are coupled to $(w_i)_{i\leq t}$ and satisfy $v_i\in S$,
  and then separately apply concentration to $(v_i)_{i\leq t}$ and derive
  $v_i = w_i \in S$ and a risk bound on both --- but will additionally apply concentration
  to $\|x_iy_i\|_*^2$.  To this end and to simplify a few terms, define $C_3 := C_1 + C_2 B_w + B_w$,
  whereby \Cref{fact:quad_bound} grants, for any $j<t$ and any $v\in S$,
  \begin{align}
    \|\partial \ell_{j+1}(v)\| \leq Z_{j+1} C_3,
    \qquad
    \envert{ \ell_{j+1}(\wref) - \ell_{j+1}(v)}
    \leq Z_{j+1}^2 C_3 B_w.
    \label{eq:qb:2}
  \end{align}
  
  To this end, the first step is to use one of the two assumptions to control
  $\sum_{j=1}^i \max\{1, \|x_j\|_*^4, |y_j|^4\}$.
  \begin{enumerate}
    \item
      \textbf{(Subgaussian tails.)}
      Union bounding over $2t$ standard subgaussian bounds \citep{vanhandel},
      simultaneously for every $i\leq t$,
      \begin{align*}
\sum_{j=1}^i Z_j 
        &\leq t \Ex Z_1 + 2\sigma \sqrt{t \ln(1/\delta)}
=: t C_7
        .
      \end{align*}

    \item
      \textbf{(Polynomial tails.)}
      Union bounding over $2t$ applications of \Cref{fact:conc:poly},
      simultaneously for every $i\leq t$,
      \begin{align*}
\sum_{j=1}^i Z_j
        &\leq t \Ex Z_1 + 2M\sqrt{t} \del{\frac 2\delta}^{1/p}
=:
        t C_8
        .
\end{align*}

  \end{enumerate}
  The rest of the proof will simply use $C$ to denote either $C_7$ or $C_8$, and the final
  bounds will be obtained by using the appropriate setting to expand the definition of $C$.

  Next comes the concentration inequality on
  $\sum_{j<i}\sbr{ \ell_{j+1}(\wref) - \ell_{j+1}(v_j)}$ for all $i\leq t$.
  It will not be possibly to apply Azuma's inequality directly, since the increments do
  not have a uniform control; instead, 
  a very nice extension of Azuma's inequality,
  presented by \citep[Problem 3.11]{vanhandel}, will allow us to use
  the varying increments which were controlled with high probability above.
  In particular, combining the above moment bounds with
  the loss bounds from \cref{eq:qb:2}
  gives
  \begin{align*}
    \sum_{j<i}\envert{ \ell_{j+1}(\wref) - \ell_{j+1}(v_j) }^2
    &\leq
    \sum_{j<i} Z_{j+1}^4 C_3^2 B_w^2
    \leq
    16 t C C_3^2 B_w^2
    .
  \end{align*}
  As such, applying the variant of Azuma's inequality from \citep[Problem 3.11]{vanhandel}
  to each $i\leq t$ and union bounding,
  and using the earlier control on the data norms to remove the ``and'' case from
  the bound in \citep[Problem 3.11]{vanhandel},
  then with probability at least $1-t\delta$, simultaneously for every $i\leq t$,
  \begin{align*}
    \sum_{j<i}\sbr{ \ell_{j+1}(\wref) - \ell_{j+1}(v_j) - \cR(\wref) + \cR(v_j)}
    &\leq
    \sqrt{\frac{16 t C C_3^2 B_w^2 \ln(1/\delta)}{2}}
    \\
    &
    \leq
    4 C_3 B_w \sqrt{t C\ln(1/\delta)}.
  \end{align*}
  This completes the expanded concentration part of the proof technique.

  The induction part now proceeds as usual.  The base case has $w_0 = v_0 \in S$ by the
  initial conditions,
  thus consider $i>0$.
  By the deterministic mirror descent guarantee in \Cref{fact:md} and since $(w_j)_{j<i}
  = (v_j)_{j<i}$,
  and also controlling the gradient norm via \cref{eq:qb:2},
  \begin{align*}
    D_\psi(\wref, w_i)
    &\leq
    D_\psi(\wref,w_0)
    + \eta \sum_{j<i} \sbr{ \ell_{j+1}(\wref) - \ell_{j+1}(w_j) }
    + \sum_{j<i}\frac {\eta^2}{2} \|\partial \ell_{j+1}(w_j)\|_*^2
    \\
    &\leq
    D_\psi(\wref,w_0)
    + \eta \sum_{j<i} \sbr{ \ell_{j+1}(\wref) - \ell_{j+1}(v_j) }
    + \frac {\eta^2 C_3^2}{2} \sum_{j<i}Z_{j+1}^2 
    \\
    &\leq
    \frac {B_w^2}{16}
    + \eta \sum_{j<i} \sbr{ \cR(\wref) - \cR(v_j) }
    + 4\eta C_3 B_w  \sqrt{t C\ln(1/\delta)}
    + \eta^2 C_3^2 t C
    \\
    &\leq
    \frac {B_w^2}{8}
    + \eta \sum_{j<i} \sbr{ \cR(\wref) - \cR(v_j) },
  \end{align*}
  which establishes the risk guarantee for $(w_j)_{j<i}$ after substituting
  $(v_j)_{j<i} = (w_j)_{j<i}$ back in.
  To see that the projection is not invoked and in fact $v_i = w_i$,
  then using the bound
$\cE(\wref) \leq D_\psi(\wref,w_0)/\sqrt{t}$,
  the preceding simplifies further to give
  \begin{align*}
    \frac 1 2 \|\wref - w_i\|^2
    \leq
    D_\psi(\wref, w_i)
    &\leq
    \frac {B_w^2}{8}
    + \eta \sum_{j<i} \sbr{ \cR(\wref) - \cR(v_j) }
    \leq
\frac {B_w^2}{4},
  \end{align*}
  meaning the projection set is not exceeded, and $v_i = w_i$.
\end{proofof}

\subsection{Batch data: proofs of \Cref{fact:md:batch,fact:md:batch}}

The proof of \Cref{fact:md:batch} follows the same scheme as \Cref{fact:md:general},
but with the Markov chain concentration inequality in \Cref{fact:conc:markov} replaced
with the generalization bound at the end of \Cref{fact:quad_bound}.

\begin{proofof}{\Cref{fact:md:batch}}
  Let $(v_i)_{i\leq t}$ denote full batch projected mirror descent iterates onto a constraint
  set $S$ using initial condition $v_0 = w_0$ and the same sample as $(w_i)_{i \leq t}$.
  As usual, the constraint set is
  \[
    S := \cbr{ v\in\R^d : \|v - \wref\|\leq B_w}.
  \]
  To control deviations over $S$, rather than a concentration inequality, the generalization
  bound at the end of \Cref{fact:quad_bound} will be used;
  in particular, defining $C_3 := C_1 + 2C_2 B_w + C_2 \|\wref\|$ for convenience,
  it holds unconditionally for every $u,v\in S$ that
  \begin{align*}
    \envert{\ell_j(y, x^\T v) - \ell_j(y, x^\T \wref)}
   &
    \leq
    C_3 B_w,
    \\
\envert{\ell'(y, x^\T v)}
   &
    \leq
    C_3,
    \\
    \enVert{ \partial \hcR(v_j) }_*
    &\leq
    \frac 1 n \sum_{i<n} \enVert{ \partial \ell_i(v_j) }_*
    \leq
    C_3,
  \end{align*}
  and with probability at least $1-4\delta$, simultaneously for every $u\in S$,
  \[
    \envert{\hcR(\wref) - \hcR(u) - \cR(\wref) + \cR(u)}
    \leq
    \frac {B_0C_3}{\sqrt n}
    \del{
      C_6 + 6\sqrt{\ln(1/\delta)}
    }.
  \]

  The proof now proceeds by induction.  The base case $w_0=v_0 \in S$ is direct,
  thus consider $i>0$.
  By the deterministic mirror descent guarantee applied to $w_i$ but now with batch
  gradients $\nabla\hcR(w_i)$,
  and using $(w_j)_{j<i} = (v_j)_{j<i}$,
and lastly using $t\leq n$,
  \begin{align*}
    D_\psi(\wref, w_i)
    &\leq
    D_\psi(\wref,w_0)
    + \eta \sum_{j<i} \sbr{\hcR(\wref) - \hcR(w_j)}
    + \eta^2 \sum_{j<i} \enVert{\nabla \hcR(w_j)}^2_*
    \\
    &\leq
    \frac {B_w^2}{16}
    + \eta \sum_{j<i} \sbr{\hcR(\wref) - \hcR(v_j)}
    + \eta^2 \sum_{j<i} \enVert{\nabla \hcR(v_j)}^2_*
    \\
    &\leq
    \frac {B_w^2}{16}
    + \eta \sum_{j<i} \sbr{\cR(\wref) - \cR(v_j)}
    + \eta C_3 B_w\sqrt{t}
    \del{
      C_6 + 6\sqrt{\ln(1/\delta)}
    }
    + t\eta^2 C_3^2
    \\
    &\leq
    \frac {B_w^2}{8}
    + \eta \sum_{j<i} \sbr{\cR(\wref) - \cR(v_j)},
  \end{align*}
  which establishes the risk guarantee on $(w_j)_{j<i}$
  after using $(v_j)_{j<i} = (w_j)_{j<i}$.
  For $w_i = v_i\in S$,
  continuing from the preceding inequality but
  additionally making use of $\cE(\wref) \leq D_\psi(\wref,w_0) / \sqrt{t}$,
  \begin{align*}
    D_\psi(\wref, w_i)
    &\leq
    \frac {B_w^2}{8}
    + \eta \sum_{j<i} \sbr{\cR(\wref) - \cR(v_j)},
    \\
    &\leq
    \frac {B_w^2}{8}
    + \eta \sum_{j<i} \frac {D_\psi(\wref,w_0)} {\sqrt{t}}
    \\
    &\leq
    \frac {B_w^2}{4},
  \end{align*}
  as desired.
\end{proofof}

In order to handle MF, the first step is a deterministic MF guarantee analogous to
the deterministic MD bound in \Cref{fact:md}.

\begin{lemma}\label[lemma]{fact:mf}
  Let a Legendre mirror map $\psi$ and objective $f$ be given,
  and let $((w_s, g_s, q_s))_{s\geq 0}$ be given by MF as in \cref{eq:mf}.
  Then, for any $t$ and any reference point $\wref$,
  \[
    D_\psi(\wref, w_{t})
    =
    D_\psi(\wref, w_{0})
    +
    \int_0^t \ip{\wref - w_s}{g_s} \dif s.
  \]
  Furthermore, if $f$ is convex, then
  \[
    D_\psi(\wref, w_{t})
    +
    \int_0^t f(w_s) \dif s
    \leq
    D_\psi(\wref, w_{0})
    +
    t f(\wref).
  \]
\end{lemma}
\begin{proof}
  This proof spiritually follows the same scheme as \Cref{fact:md}, though
  somewhat inside-out: rather than considering one time step and applying a summation,
  it starts from the fundamental theorem of calculus and then adjusts individual time steps.
  Moreover, since it does not need to deal with the squared gradient term, overall it is simpler.
  In detail, using the Fenchel-Young inequality and the fundamental theorem of calculus,
  \begin{align*}
    D_\psi(\wref, w_t) - D_\psi(\wref, w_0)
    &=
    \sbr{ -\psi(w_t) + \ip{\nabla \psi(w_t)}{w_t} }
    + \sbr{ \psi(w_0) - \ip{\nabla \psi(w_0)}{w_0} }
    \\
    &\quad
    + \ip{\wref}{\nabla \psi(w_0) - \nabla \psi(w_t)}
    \\
    &=
    \psi^*(q_t)
    - \psi^*(q_0)
    + \ip{\wref}{q_0 - q_t} 
    \\
    &=
    \int_0^t
    \frac \dif {\dif s} \sbr{\psi^*(q_s) - \ip{\wref}{q_s} }\dif s
    \\
    &=
    \int_0^t
    \sbr{\ip{\nabla\psi^*(q_s)}{\dot q_s} - \ip{\wref}{\dot q_s} }\dif s
    \\
    &=
    \int_0^t
    \ip{\wref - w_s}{g_s}
    \dif s
    ,
  \end{align*}
  which gives the first statement after rearranging.  For the second statement,
  applying convexity within the preceding integral gives
  \[
    \int_0^t
    \ip{\wref - w_s}{g_s}
    \dif s
    \leq
    \int_0^t
    \sbr{f(\wref) - f(w_s)}
    \dif s
    =
    t f(\wref)
    -
    \int_0^t
    f(w_s)
    \dif s.
  \]
\end{proof}

With \Cref{fact:mf} in hand, the proof of \Cref{fact:mf:batch}.  As discussed in the body,
it could be proved with the coupling-based proof technique, but a more direct argument
is provided instead, and made possible via the generalization bound in \Cref{fact:quad_bound}.

\begin{proof}[Proof of \Cref{fact:mf:batch}]
  The first part of the proof is as in \Cref{fact:md:batch}:
  define the bal $S$ as
  \[
    S := \cbr{ v\in\R^d : \|v - \wref\|\leq B_w},
  \]
  and invoke \Cref{fact:quad_bound} and define $C_3 := C_1 + 2C_2 B_w + C_2 \|\wref\|$,
  and with probability at least $1-4\delta$,
  making use of the
  choice of $B_w$, every $u\in S$ satisfies
  \begin{align*}
    \envert{\hcR(\wref) - \hcR(u) - \cR(\wref) + \cR(u)}
    &\leq
    \frac {B_w \del{C_1 + 2C_2 B_w + C_2 \|\wref\|}}{\sqrt n}
    \del{
      C_6 + 6\sqrt{\ln(1/\delta)}
    }
    \leq
    \frac {B_w^2}{4t}
;
  \end{align*}
  henceforth, discard the failure event associated with this inequality.
  After this point, the proof differs from that of \Cref{fact:md:batch}.
  As discussed in \Cref{sec:md:batch},
  rather than using the coupling argument, a direct proof is given.

  Consider any time $\tau$ so that $w_s \in S$ for all $s\in[0,\tau]$.
  Then, by \Cref{fact:mf} and the preceding bounds,
  \begin{align*}
    D_\psi(\wref, w_{\tau})
    -
    D_\psi(\wref, w_{0})
    &\leq
    \int_0^\tau\sbr{ \hcR(\wref)  - \hcR(w_s) }  \dif s
    \\
    &\leq
    \int_0^\tau\sbr{ \cR(\wref)  - \cR(w_s) }  \dif s
    + 
    \frac {\tau B_w^2 }{4t}
;
  \end{align*}
  the proof is complete if we can choose $\tau = t$.

  To this end, let $\tau$ denote the first time with
  $D_\psi(\wref, w_\tau) = B_w^2/2$,
  and suppose contradictorily that $\tau \leq t$.
  This implies $\|\wref - w_\tau\| \leq B_w$ by strong convexity of $\psi$,
  and thus thus $w_\tau \in S$,
  but moreover since $\tau$ is defined as the first time when this holds,
  then $w_s \in S$ for $s\in [0,\tau)$ as well, whereby the preceding inequalities
  and the condition $\cE(\wref)\leq D_\psi(\wref,w_0)/\sqrt{n}$ gives
  \begin{align*}
    D_\psi(\wref, w_{\tau})
    &\leq
    D_\psi(\wref, w_0)
    +
    \int_0^\tau\sbr{ \cR(\wref)  - \cR(w_s) }  \dif s
    + 
    \frac {\tau B_w^2}{4\sqrt{n}}
    \\
    &\leq
    \frac {B_w^2}{16}
    +
    \frac {\tau D_\psi(\wref,w_0)}{\sqrt{n}}
    + 
    \frac {B_w^2}{4}
    \\
    &\leq
    \frac {3 B_w^2}{8},
  \end{align*}
  a contradiction, thus $\tau > t$.
\end{proof}

\section{Analysis of Temporal Difference learning (TD)}

As with the proof schemes for mirror descent,
there is both a deterministic part (provided for mirror descent in \Cref{fact:md}),
and a random part (using \Cref{fact:conc:markov} for concentration of Markov chains).

\subsection{Deterministic TD analysis}

Even though TD is not a gradient-based method, the analysis here follows the same
expand-the-square plan as described for gradient descent in \Cref{sec:proof_sketch},
which is also the idea behind the mirror descent bound in \Cref{fact:md}.

\begin{lemma}\label[lemma]{fact:td:det}
  Let $S \subseteq \R^d$ denote an arbitrary closed convex constraint set,
  let $\wref\in S$ be arbitrary and $w_0 \in S$,
  and given any vectors $(x_i)_{i\geq 0}$ with $\|x_i\|\leq 1$ and scalars $(r_i)_{i\geq 1}$
  with $|r_i|\leq 1$,
  consider the corresponding projected TD iterates $w_{i+1} := \Pi_S\del{w_i - \eta_{i+1} G_{i+1}(w_i)}$
  (where $G_{i+1}(\cdot)$ is defined in \cref{eq:td:G}).
  Then,
for any $t$,
\begin{align*}
    \|w_{t} - \wref\|^2
    \leq
    \|w_0 - \wref\|^2
    + \eta \sum_{i<t} \Big[&
    - \ip{x_i}{w_i - \wref}^2 + \ip{\gamma x_{i+1}}{w_i-\wref}^2
    \\
    &
  - 2\ip{G_{i+1}(\wref)}{w_i - \wref} + 4\eta \|G_{i+1}(\wref)\|^2 \Big].
  \end{align*}
\end{lemma}
\begin{proof}
  Proceeding just as in mirror descent, first note for any $i$ that
  \begin{align*}
    \|w_{i+1} - \wref\|^2
    &=
    \enVert{\Pi_S( w_{i} - \eta_{i+1} \td_{i+1}(w_i) ) - \wref }^2
    \\
    &\leq
    \enVert{ w_{i} - \eta_{i+1} \td_{i+1}(w_i) - \wref }^2
    \\
    &=
    \|w_i - \wref\|^2
    - 2\eta_{i+1} \ip{\td_{i+1}(w_i)}{w_i - \wref} + \eta_i^2 \|\td_{i+1}(w_i)\|^2.
  \end{align*}
  To simplify the two latter terms,
  since $\td_{i+1}(w_i) - \td_{i+1}(\wref) = x_i \ip{x_i - \gamma x_{i+1}}{w_i - \wref}$,
  note firstly that 
  \begin{align*}
    - \ip{\td_{i+1}(w_i)}{w_i - \wref}
    &=
    - \ip{\td_{i+1}(w_i)-\td_{i+1}(\wref)}{w_i - \wref} - \ip{\td_{i+1}(\wref)}{w_i - \wref}
    \\
    &=
    - \ip{x_i}{w_i - \wref}\ip{x_i-\gamma x_{i+1}}{w_i - \wref} - \ip{\td_{i+1}(\wref)}{w_i - \wref}
    \\
    &=
    - \ip{x_i}{w_i - \wref}^2 + \ip{x_i}{w_i - \wref}\ip{\gamma x_{i+1}}{w_i - \wref} - \ip{\td_{i+1}(\wref)}{w_i - \wref}
    ,
\end{align*}
 and secondly
 \begin{align*}
    \frac 1 2 \|\td_{i+1}(w_i)\|^2
    &\leq
    \|\td_{i+1}(w_i) - \td_{i+1}(\wref)\|^2 + \|\td_{i+1}(\wref)\|^2
    \\
    &=
    \|x_i\|^2 \ip{x_i - \gamma x_{i+1}}{w_i - \wref}^2
    + 2 \|\td_{i+1}(\wref)\|^2
    \\
    &\leq
    \ip{x_i - \gamma x_{i+1}}{w_i - \wref}^2
    + 2 \|\td_{i+1}(\wref)\|^2
    \\
    &=
    \ip{x_i}{w_i - \wref}^2
    - 2 \ip{x_i}{w_i - \wref} \ip{\gamma x_{i+1}}{w_i - \wref}
    \\
    &\qquad
    + \ip{\gamma x_{i+1}}{w_i - \wref}^2
    + 2 \|\td_{i+1}(\wref)\|^2
,
\end{align*}
 which together with $2\ip{x_i}{w_i-\wref}\ip{\gamma x_{i+1}}{w_i-\wref}
 \leq \ip{x_i}{w_i-\wref}^2 + \ip{\gamma x_{i+1}}{w_i-\wref}^2$
 give
 \begin{align*}
    &\hspace{-2em}- 2 \ip{\td_{i+1}(w_i)}{w_i - \wref}
    + \eta \|\td_{i+1}(w_i)\|^2
    \\
    &\leq
    - 2(1-\eta) \ip{x_i}{w_i - \wref}^2 + 2(1-2\eta)\ip{x_i}{w_i - \wref}\ip{\gamma x_{i+1}}{w_i - \wref}
    \\
    &\qquad +2\eta \ip{\gamma x_{i+1}}{w_i-\wref}^2 - 2\ip{\td_{i+1}(\wref)}{w_i - \wref} + 4\eta \|\td_{i+1}(\wref)\|^2
    \\
    &\leq
    - \ip{x_i}{w_i - \wref}^2 + \ip{\gamma x_{i+1}}{w_i-\wref}^2
    \\
    &\qquad - 2\ip{\td_{i+1}(\wref)}{w_i - \wref} + 4\eta \|\td_{i+1}(\wref)\|^2.
  \end{align*}
  Combining all the inequalities so far gives
\begin{align*}
    \|w_{i+1} - \wref\|^2
    -
    \|w_i - \wref\|^2
    &\leq
    - \eta \ip{x_i}{w_i - \wref}^2 + \eta \ip{\gamma x_{i+1}}{w_i-\wref}^2
    \\
    &\qquad - 2\eta\ip{\td_{i+1}(\wref)}{w_i - \wref} + 4\eta^2 \|\td_{i+1}(\wref)\|^2,
  \end{align*}
  while applying $\sum_{i<t}$ to both sides and rearranging gives the final overall inequality.
\end{proof}

\subsection{Stochastic TD analysis}

As mentioned in the body, the underlying Markov chain is $(x_i)_{i\leq t}$:
the distribution of $x_{i+1}$ is wholly determined by $x_i$.  Moreover, there are additional
scalars $(r_i)_{i=1}^n$, where the distribution of $r_{i+1}$ is also fully specified given
$x_i$, but meanwhile $r_{i+1}$ says nothing about $x_{i+1}$.
Confusing matters, the TD update makes use of a triple $\zeta = (x_{i}, x_{i+1}, r_{i+1})$;
in particular, define
\[
  \td_\zeta(v) = x \del{\ip{x - \gamma x'}{v} - r}
  \quad\textup{and}\quad
  \td_{i+1}(v) = x_i \del{\ip{x_i - \gamma x_{i+1}}{v} - r_{i+1}}.
\]

The underlying mixing assumptions will be placed on $(\zeta_i)_{i\leq t}$,
not on $(x_i)_{i\leq t}$; this is simply to avoid messiness arising from
the use of multiple indexes.  This is all smoothed over with the approximate stationarity
of \Cref{defn:witness}.
That said, there is a place in the proof where the Markov structure
on $(x_i)_{i\leq t}$ is needed (and explicitly mentioned).

\begin{proofof}{\Cref{fact:td}}
  Following the standard proof structure,
  let $(v_i)_{i\leq t}$
  denote projected TD iterates constrained to lie 
  within $S := \{ v\in\R^d : \|v-\wref\|\leq B_w\}$.
  The $(v_i)_{i\leq t}$ are coupled to $(w_i)_{i \leq t}$ in the strong sense
  used throughout this work: $v_0 := w_0$,
  and thereafter both are updated using the exact same random data sequence
  $(\zeta_i)_{i\leq t}$.
  The next step of the proof will be to apply concentration inequalities to control
  the underlying sample space $(\zeta_i)_{i\leq t}$, and then use
  these to show that in fact $w_i=v_i\in S$ via the deterministic TD guarantees
  in \Cref{fact:td:det}.

  The concentration inequality will be the one for stochastic processes
  in \Cref{fact:conc:markov},
  which first requires a variety of uniform bounds.
  For convenience, define $j := i -\tau + 1$
  and a mapping $f$ as
  \begin{align*}
    f(\zeta;v)
    &:= 
    - \ip{x}{v - \wref}^2 + \ip{\gamma x'}{v-\wref}^2
    - 2\ip{\td_\zeta(\wref)}{v - \wref} + 4\eta \|\td_\zeta(\wref)\|^2;
  \end{align*}
  the uniform controls will be on this mapping $f$, which corresponds to the $f$
  in \Cref{fact:conc:markov},
  and is the stochastic term in \cref{fact:td:det}.
  The first step is to bound $|f|$, which is direct:
  using $\eta \leq 1$ and $\gamma \leq 1$,
  \begin{align*}
    \envert{f(\zeta;v)}
    &\leq
    \|x\|^2 \|v-\wref\|^2 + \gamma^2 \|x'\|^2 \|v-\wref\|^2
    \\
    &\qquad
    +
    \|\td_\zeta(\wref)\|^2 + \|v-\wref\|^2
    + 4\eta \|x\|^2 \del{\ip{x-\gamma x'}{\wref} - r}^2
    \\
    &\leq
    B_w^2 + \gamma^2 B_w^2 
    + B_w^2 
    + (1+4\eta)(2 + (2+2\gamma^2)\|\wref\|^2)
    \\
    &\leq
    30 (B_w^2 + \|\wref\|^2)
    \\
    &\leq
    31 B_w^2 =: B_f,
  \end{align*}
  To bound $|f(\zeta_{i+1};v_i) - f(\zeta_{i+1};v_j)|$, first note that
  \begin{align*}
    \|v_i - v_j\|
    &= \enVert{ \sum_{k=j}^{i-1} \eta \del{ \td_{k+1}(v_k) - \td_{k+1}(\wref) + \td_{k+1}(\wref)}  }
    \\
    &= \sum_{k=j}^{i-1} \eta \enVert{x_k \ip{x_k -\gamma x_{k+1}}{v_k - \wref}
    + x_k\del{ \ip{x_k - \gamma x_{k+1}}{\wref} - r_{k+1}}}
    \\
    &\leq
    2 \eta \del{\tau - 1}\del{1 + B_w + \|\wref\|}
    \\
    &\leq
    6 \eta B_w =: B_G,
  \end{align*}
  whereby
  \begin{align*}
    \envert{f(\zeta_{i+1};v_{i}) - f(\zeta_{i+1};v_{j})}
    &=
    \Big|
    - \ip{x_i}{v_i - v_j + v_j - \wref}^2 + \ip{x_i}{v_j-\wref}^2
    \\
    &\quad
    + \ip{\gamma x_{i+1}}{v_i - v_j + v_j - \wref}^2 - \ip{\gamma x_{i+1}}{v_j-\wref}^2
    \\
    &\quad
    + 2\ip{\td_\zeta(\wref)}{v_j - v_i}
    \Big|
    \\
    &=
    \Big|
    - \ip{x_i}{v_i - v_j}^2 - 2\ip{x_i}{v_i-v_j}\ip{x_i}{v_j - \wref}
    \\
    &\quad
    + \gamma^2\ip{x_{i+1}}{v_i - v_j}^2 + 2\gamma^2\ip{x_{i+1}}{v_i - v_j}\ip{x_{i+1}}{v_j-\wref}
    \\
    &\quad
    + 2\ip{x_i}{v_j - v_i}\ip{x_{i}-\gamma x_{i+1}}{\wref} - 2 r_{i+1}\ip{x_i}{v_j - v_i}
    \Big|
    \\
    &\leq B_G^2 + 2 B_w B_{\td}
    + \gamma^2 B_{\td}^2 + 2\gamma^2 B_{\td} B_w + 2 (1+\gamma)B_\td \|\wref\| + 2 B_{\td}
    \\
&\leq
    B_G \del{ 2 + 4\|\wref\| + 4 B_w } + 2 B_G^2
    \\
    &\leq
    42 \eta B_w^2 + 36\eta^2 B_w^2
    \\
    &
    \leq 80 \eta B_w^2
    =: B_i.
  \end{align*}
  Applying \Cref{fact:conc:markov} to each $i\leq t$ and union bounding,
  then with probability at least $1-t\delta$,
  simultaneously for every $i\leq t$,
  \begin{align*}
    \eta \sum_{j<i}\sbr{ f(\zeta_{j+1}; v_j)
    -
    \Ex_{\zeta\sim \pi} f(\zeta;v_j)}
    &
    \leq
    2\eta B_f \del{ 2\tau - 2 + \sqrt{t} + \sqrt{t\tau \ln(1/\delta)}}
    + \eta\sum_{i<t} B_i
    \\
    &\leq
    \eta B_w^2 \del{
    62(2\tau - 2 + \sqrt{t}) + 80t\eta + \sqrt{2t\tau \ln(1/\delta)}}
    \\
    &\leq
    \frac {B_w^2}{1024}\del{124 + 80+ 2}
    \leq
    \frac {B_w^2}{4}.
\end{align*}
  Henceforth, condition away the failure event for the preceding inequalities.

  What remains is to inductively invoke the deterministic TD guarantee from \Cref{fact:td:det}
  to bound the error of $(w_i)_{i\leq t}$ and simultaneously show $w_i = v_i \in S$ for all $i$.

  Throw out the preceding failure event; the remainder of the proof proceeds by induction,
  establishing that the iterate sequence never exits $S$, 
  The proof now proceeds by induction; the claim holds automatically for $v_0$, since
  $v_0 = w_0 \in S$ by construction, thus consider $w_i$ for some $i>0$.
  Invoking the deterministic TD guarantee from \Cref{fact:td:det} to $w_i$,
  together with the inductive hypothesis $(w_j)_{j<i} = (v_j)_{j<i}$,
  the earlier concentration inequalities,
  and lastly 
  making the single appeal to the Markov property on $(x_i)_{i\leq t}$
  to obtain $\Ex_{\zeta_{i+1}} x_{i+1} = \Ex_{\zeta_{i+1}} x_i$,
  note
  \begin{align}
    \|\tilde w_{i} - \wref\|^2
    &\leq
    \|w_0 - \wref\|^2
    + \eta \sum_{j<i} \Big[
    -  \ip{x_j}{w_j - \wref}^2 +  \ip{\gamma x_{j+1}}{w_j-\wref}^2
    \notag\\
    &\qquad - 2\ip{\td_j(\wref)}{w_j - \wref} + 4\eta \|\td_j(\wref)\|^2 \Big]
    \notag\\
    &\leq
    \|w_0 - \wref\|^2
    + \eta \sum_{j<i} \Big[
    -  \ip{x_j}{v_j - \wref}^2 +  \ip{\gamma x_{j+1}}{v_j-\wref}^2
    \notag\\
    &\qquad - 2\ip{\td_j(\wref)}{v_j - \wref} + 4\eta \|\td_j(\wref)\|^2 \Big]
    \notag\\
    &\leq
    \|w_0 - \wref\|^2
    + \frac {B_w^2}{4}
    + \eta \sum_{j<i} \Ex_{\zeta\sim\pi} \Big[
    - \ip{x}{v_j - \wref}^2 + \ip{\gamma x}{v_j-\wref}^2
    \notag\\
    &\qquad - 2\ip{\td_\zeta(\wref)}{v_j - \wref} + 4\eta (2 + 2(1+\gamma^2)\|\wref\|^2) \Big]
    \notag\\
    &\leq
    \|w_0 - \wref\|^2
    + \frac {B_w^2}{2}
    \notag\\
    &\qquad
    + \eta \sum_{j<i} \Ex_{\zeta\sim\pi} \Big[
    - (1-\gamma^2) \ip{x}{w_j - \wref}^2 - 2\ip{\td_\zeta(\wref)}{w_j - \wref} \Big]
\label{eq:td:mid:1}
    ,
  \end{align}
  which rearranges to give the desired TD error bound.
  To control the norms from here, since $\|\Ex_{\zeta\sim\pi} \td_{\zeta}(\wref)\| \leq
  \frac {\|\wref - w_0\|^2} {\sqrt{t}}$, and moreover since $-(1-\gamma^2)\ip{x}{v_i - \wref}^2$ is negative
  and can be dropped, then \cref{eq:td:mid:1} implies
  \begin{align*}
    \|w_{i} - \wref\|^2
    &\leq
    \|v_0 - \wref\|^2
    + \frac {B_w^2}{2}
    + 2 \eta \sum_{j<i} \enVert{\td_\pi(\wref)}\enVert{v_j - \wref}
    \\
    &\leq
    \frac {B_w^2}{16}
    + \frac{B_w^2}{2}
    + \frac {B_w} 8
    \\
    &< B_w^2,
  \end{align*}
  meaning the new unconstrained iterate $w_i$ satisfies $w_i\in S$, whereby
  $v_i$ will also not encounter the constraint since $v_{i-1} = w_{i-1}$ and the update
  is the same, and thus $w_i = v_i \in S$.
\end{proofof}

\end{document}